\pdfoutput=1

\documentclass{article}
\usepackage{iclr2021_conference,times}

\usepackage{amsfonts}       %
\usepackage{hyperref}
\usepackage{url}
\usepackage{graphicx}
\usepackage{amsmath}
\usepackage{amssymb}
\usepackage{algorithm}
\usepackage[noend]{algpseudocode}
\usepackage{amsthm}
\usepackage{cancel}
\usepackage{color}
\usepackage{xcolor}
\usepackage{wrapfig}
\usepackage{caption}
\usepackage{subcaption}
\usepackage{mathtools}
\usepackage{bbm}
\usepackage[normalem]{ulem}
\usepackage{nccmath}

\DeclarePairedDelimiterX{\infdivx}[2]{(}{)}{%
  #1\;\delimsize\|\;#2%
}
\newcommand{\kl}{D_{\mathrm{KL}}\infdivx}

\newtheorem{assumption}{Assumption}
\newtheorem{definition}{Definition}
\newtheorem{theorem}{Theorem}[section]

\newtheorem{lemma}[theorem]{Lemma}

\newcommand{\figleft}{{\em (Left)}}

\newcommand{\figright}{{\em (Right)}}

\def\eqref#1{equation~\ref{#1}}

\def\1{\bm{1}}

\DeclareMathAlphabet{\mathsfit}{\encodingdefault}{\sfdefault}{m}{sl}
\SetMathAlphabet{\mathsfit}{bold}{\encodingdefault}{\sfdefault}{bx}{n}

\def\gA{{\mathcal{A}}}

\def\gD{{\mathcal{D}}}

\def\gH{{\mathcal{H}}}

\def\gM{{\mathcal{M}}}
\def\gN{{\mathcal{N}}}
\def\gO{{\mathcal{O}}}

\def\gS{{\mathcal{S}}}

\def\gU{{\mathcal{U}}}

\newcommand{\E}{\mathbb{E}}

\newcommand{\Var}{\mathrm{Var}}

\DeclareMathOperator*{\argmax}{arg\,max}

\title{Off-Dynamics Reinforcement Learning:\\ \Large Training for Transfer with Domain Classifiers}

\author{%
  Benjamin Eysenbach\thanks{Equal contribution.} \\
  CMU, Google Brain\\
  \texttt{beysenba@cs.cmu.edu} \\
  \And
  Shreyas Chaudhari$^*$ \\
  CMU \\
  \texttt{shreyaschaudhari@cmu.edu} \\
  \And
  Swapnil Asawa$^*$ \\
  University of Pittsburgh \\
  \texttt{swa12@pitt.edu} \\
  \AND
  Sergey Levine \\
  UC Berkeley, Google Brain
 \And
  Ruslan Salakhutinov \\
  CMU \\
}

\iclrfinalcopy %
\begin{document}

\maketitle

\begin{abstract}
We propose a simple, practical, and intuitive approach for domain adaptation in reinforcement learning. Our approach stems from the idea that the agent's experience in the source domain should look similar to its experience in the target domain. Building off of a probabilistic view of RL, we achieve this goal by compensating for the difference in dynamics by modifying the reward function. This modified reward function is simple to estimate by learning auxiliary classifiers that distinguish source-domain transitions from target-domain transitions. Intuitively, the agent is penalized for transitions that would indicate that the agent is interacting with the source domain, rather than the target domain. Formally, we prove that applying our method in the source domain is guaranteed to obtain a near-optimal policy for the target domain, provided that the source and target domains satisfy a lightweight assumption. Our approach is applicable to domains with continuous states and actions and does not require learning an explicit model of the dynamics. On discrete and continuous control tasks, we illustrate the mechanics of our approach and demonstrate its scalability to high-dimensional~tasks.
\end{abstract}

\section{Introduction}

Reinforcement learning (RL) can automate the acquisition of complex behavioral policies through real-world trial-and-error experimentation. However, many domains where we would like to learn policies are not amenable to such trial-and-error learning, because the errors are too costly: from autonomous driving to flying airplanes to devising medical treatment plans, safety-critical RL problems necessitate some type of \emph{transfer learning}, where a safer source domain, such as a simulator, is used to train a policy that can then function effectively in a target domain. In this paper, we examine a specific transfer learning scenario that we call domain adaptation, by analogy to domain adaptation problems in computer vision~\citep{csurka2017domain}, where the training process in a source domain can be modified so that the resulting policy is effective in a given target domain.

\begin{wrapfigure}[13]{r}{0.5\textwidth}
    \centering
    \vspace{-1.em}
    \includegraphics[width=\linewidth]{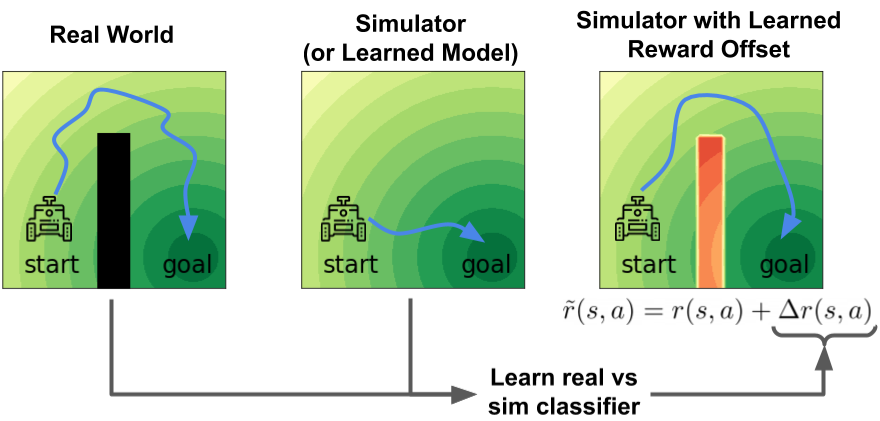}
    \vspace{-1.6em}
    \caption{\footnotesize Our method acquires a policy for the target domain by practicing in the source domain using a (learned) modified reward function.} %
    \label{fig:teaser}
\end{wrapfigure}
RL algorithms today require a large amount of experience in the \emph{target domain}.
However, for many tasks
we may have access to a different but structurally similar \emph{source domain}. While the source domain has different dynamics than the target domain, experience in the source domain is much cheaper to collect.
However, transferring policies from one domain to another is challenging because strategies which are effective in the source domain may not be effective in the target domain.
For example, aggressive driving may work well on a dry racetrack but fail catastrophically on an icy road.
While prior work has studied the domain adaptation of \emph{observations} in RL~\citep{bousmalis2018using, ganin2016domain, higgins2017darla}, it ignores the domain adaptation of the \emph{dynamics}.%

This paper presents a simple %
approach for domain adaptation in RL, illustrated in Fig.~\ref{fig:teaser}.
Our main idea is that the agent's experience in the source domain should look similar to its experience in the target domain. Building off of a probabilistic view of RL, we formally show that we can achieve this goal by \emph{compensating for the difference in dynamics by modifying the reward function}. This modified reward function is simple to estimate by learning auxiliary classifiers that distinguish source-domain transitions from target-domain transitions.
Because our method learns a classifier, rather than a dynamics model, we expect it to handle high-dimensional tasks better than model-based methods, a conjecture supported by experiments on the 111-dimensional Ant task.
Unlike prior work based on similar intuition~\citep{koos2012transferability, wulfmeier2017mutual}, a key contribution of our work is a formal guarantee that our method yields a near-optimal policy for the target domain.

The main contribution of this work is an algorithm for domain adaptation to dynamics changes in RL, based on the idea of compensating for differences in dynamics by modifying the reward function.
We call this algorithm Domain Adaptation with Rewards from Classifiers, or DARC for short.
DARC does not estimate transition probabilities, but rather modifies the reward function using a pair of classifiers. We formally analyze the conditions under which our method produces near-optimal policies for the target domain. On a range of discrete and continuous control tasks, we both illustrate the mechanics of our approach and demonstrate its scalability to higher-dimensional~tasks.

\section{Related Work}
\label{sec:related-work}

While our work will focus on domain adaptation applied to RL, we start by reviewing more general ideas in domain adaptation,
and defer to~\citet{kouw2019review} for a recent review of the field.
Two common approaches to domain adaptation are importance weighting and domain-agnostic features. \emph{Importance-weighting methods} (e.g.,~\citep{zadrozny2004learning, cortes2014domain, lipton2018detecting}) estimate the likelihood ratio of examples under the target domain versus the source domain, and use this ratio to re-weight examples sampled from the source domain.
Similar to prior work on importance weighting~\citep{bickel2007discriminative, sonderby2016amortised, mohamed2016learning, uehara2016generative}, our method will use a classifier to estimate a probability ratio.
Since we will need to estimate the density ratio of conditional distributions (transition probabilities), we will learn two classifiers. %
Importantly, we will use the logarithm of the density ratio to modify the reward function instead of weighting samples by the density ratio, which is often numerically unstable (see, e.g.,~\citet[\S3]{schulman2017proximal}) and led to poor performance in our experiments.

Prior methods for applying domain adaptation to RL include approaches based on system identification, domain randomization, and observation adaptation. Perhaps the most established approach, system identification~\citep{ljung1999system}, uses observed data to tune the parameters of a simulator~\citep{feldbaum1960dual, werbos1989neural, wittenmark1995adaptive, ross2012agnostic, tan2016simulation, zhu2017fast, farchy2013humanoid}
More recent work has successfully used this strategy to bridge the sim2real gap~\citep{chebotar2019closing, rajeswaran2016epopt}.
Closely related is work on online system identification and meta-learning, which directly uses the inferred system parameters to update the policy~\citep{yu2017preparing, clavera2018learning, tanaskovic2013adaptive, sastry1989adaptive}. However, these approaches typically require either a model of the environment or a manually-specified distribution over potential test-time dynamics, requirements that our method will lift.
Another approach, \emph{domain randomization}, randomly samples the parameters of the source domain and then finds the best policy for this randomized environment~\citep{sadeghi2016cad2rl, tobin2017domain, peng2018sim, cutler2014reinforcement}. While often effective, this method is sensitive to the choice of which parameters are randomized, and the distributions from which these simulator parameters are sampled.
A third approach, \emph{observation adaptation}, modifies the observations of the source domain to appear similar to those in the target domain~\citep{fernando2013unsupervised, hoffman2016fcns, wulfmeier2017addressing}. While this approach has been successfully applied to video games~\citep{gamrian2018transfer} and robot manipulation~\citep{bousmalis2018using}, it ignores the fact that the source and target domains may have differing dynamics.

Finally, our work is similar to prior work on transfer learning~\citep{taylor2009transfer} and meta-learning in RL, but makes less strict assumptions than most prior work. For example, most work on meta-RL~\citep{killian2017robust, duan2016rl, mishra2017simple, rakelly2019efficient} and some work on transfer learning~\citep{perkins1999using, tanaka2003multitask, sunmola2006model} assume that the agent has access to many source tasks, all drawn from the same distribution as the target task. ~\citet{selfridge1985training, madden2004transfer} assume a manually-specified curriculum of tasks, ~\citet{ravindran2004algebraic} assume that the source and target domains have the same dynamics locally, and \citet{sherstov2005improving} assume that the set of actions that are useful in the source domain is the same as the set of actions that will be useful in the target domain. Our method does not require these assumptions, allowing it to successfully learn in settings where these prior works would fail. For example, the assumption of~\citet{sherstov2005improving} is violated in our experiments with broken robots: actions which move a joint are useful in the source domain (where the robot is fully-function) but not useful in the target domain (where that joint is disabled). Our method will significantly outperform an importance weighting baseline~\citep{lazaric2008knowledge}.
Unlike~\citet{vemula2020planning}, our method does not require learning a dynamics model and is applicable to stochastic environments and those with continuous states and actions. Our algorithm bears a resemblance to that in~\citet{wulfmeier2017mutual}, but a crucial algorithmic difference allows us to prove that our method acquires a near-optimal policy in the target domain, and also leads to improved performance empirically. %

The theoretical derivation of our method is inspired by prior work which formulates control as a problem of probabilistic inference (e.g.,~\citep{toussaint2009robot, rawlik2013stochastic, levine2018learning}). %
Algorithms for model-based RL (e.g.,~\citep{deisenroth2011pilco, hafner2018learning,janner2019trust}) %
and off-policy RL (e.g.,~\citep{munos2016safe, fujimoto2018off, dann2014policy, dudik2011doubly} similarly aim to improve the sample efficiency of RL, but do use the source domain to accelerate learning. Our method is applicable to any maximum entropy RL algorithm, including on-policy~\citep{song2019v}, off-policy~\citep{abdolmaleki2018maximum, haarnoja2018soft}, and model-based~\citep{janner2019trust, williams2015model} algorithms. We will use SAC~\citep{haarnoja2018soft} in our experiments and compare against model-based baselines.

\section{Preliminaries}
\label{sec:prelim}

In this section, we introduce notation and formally define domain adaptation for RL. Our problem setting will consider two MDPs: $\gM_\text{source}$ represents the source domain (e.g., a practice facility, simulator, or learned approximate model of the target domain) while $\gM_\text{target}$ represents a the target domain. We assume that the two domains have the same state space $\gS$, action space $\gA$, reward function $r$, and initially state distribution $p_1(s_1)$; the only difference between the domains is the dynamics, $p_{\text{source}}(s_{t+1} \mid s_t, a_t)$ and $p_{\text{target}}(s_{t+1} \mid s_t, a_t)$. We will learn a Markovian policy $\pi_\theta(a \mid s)$, parametrized by $\theta$. Our objective is to learn a policy $\pi$ that maximizes the expected discounted sum of rewards on $\gM_\text{target}$, $\E_{\pi, \gM_\text{target}}[\sum_t \gamma^t r(s_t, a_t)]$. We now formally define our problem setting:

\begin{definition} \textbf{Domain Adaptation for RL} is the problem of using interactions in the source MDP $\gM_{\text{source}}$ together with a small number of interactions in the target MDP $\gM_\text{target}$ to acquire a policy that achieves high reward in the target MDP, $\gM_\text{target}$.
\end{definition}
We will assume every transition with non-zero probability in the target domain will have non-zero probability in the source domain:
\begin{equation}
    p_\text{target}(s_{t+1} \mid s_t, a_t) > 0 \implies p_\text{source}(s_{t+1} \mid s_t, a_t) > 0 \qquad \text{for all }s_t, s_{t+1} \in \gS, a_t \in \gA. \label{eq:completeness}
\end{equation}
This assumption is common in work on importance sampling~\citep[\S12.2.2]{koller2009probabilistic}, and the converse need not hold: transitions that are possible in the source domain need not be possible in the target domain. If this assumption did not hold, then the optimal policy for the target domain might involve behaviors that are not possible in the source domain, so it is unclear how one could learn a near-optimal policy by practicing in the source domain.

\section{A Variational Perspective on Domain Adaptation in RL}
\label{sec:variational}

The probabilistic inference interpretation of RL~\citep{kappen2005path, todorov2007linearly, toussaint2009robot, ziebart2010modeling, rawlik2013stochastic, levine2018reinforcement} treats the reward function as defining a desired distribution over trajectories. The agent's task is to sample from this distribution by picking trajectories with probability proportional to their exponentiated reward. This section will reinterpret this model in the context of domain transfer, %
showing that domain adaptation of \emph{dynamics} can be done by modifying the \emph{rewards}.
To apply this model to domain adaptation, define $p(\tau)$ as the desired distribution over trajectories in the target domain,
\begin{equation*}
    p(\tau) \propto p_1(s_1) \bigg(\prod_t p_{\text{target}}(s_{t+1} \mid s_t, a_t) \bigg) \exp \bigg( \sum_t r(s_t, a_t) \bigg),
\end{equation*}
and $q(\tau)$ as our agent's distribution over trajectories in the source domain,
\begin{equation*}
    q(\tau) = p_1(s_1) \prod_t p_{\text{source}}(s_{t+1} \mid s_t, a_t) \pi_\theta(a_t \mid s_t).
\end{equation*}
As noted in Section~\ref{sec:prelim}, we assume both trajectory distributions have the same initial state distribution.
Our aim is to learn a policy whose behavior in the source domain both receives high reward and has high likelihood under the target domain dynamics. We codify this objective by minimizing the reverse KL divergence between these two distributions:
\begin{equation}
    \min_{\pi(a \mid s)} \kl{q}{p}= - \E_{p_\text{source}}\bigg[ \sum_t r(s_t, a_t) + \gH_\pi[a_t \mid s_t] + \Delta r(s_{t+1}, s_t, a_t) \bigg] + c, \label{eq:kl}
\end{equation}
where
\begin{equation*}
    \Delta r(s_{t+1}, s_t, a_t) \triangleq \log p(s_{t+1} \mid s_t, a_t) - \log q(s_{t+1} \mid s_t, a_t).
\end{equation*}
The constant $c$ is the partition function of $p(\tau)$, which is independent of the policy and dynamics. 
While $\Delta r$ is defined in terms of transition probabilities, in Sec.~\ref{sec:method} we show how to estimate $\Delta r$ by learning a classifier. We therefore call our method \emph{domain adaptation with rewards from classifiers} (DARC), and will use $\pi_\text{DARC}^*$ to refer to the policy that maximizes the objective in Eq.~\ref{eq:kl}.

Where the source and target dynamics are equal, the correction term $\Delta r$ is zero and we recover maximum entropy~RL~\citep{ziebart2010modeling, todorov2007linearly}.
The reward correction is different from prior work that adds $\log \beta(a \mid s)$ to the reward to regularize the policy to be close to the behavior policy $\beta$ (e.g.,~\citet{jaques2017sequence, abdolmaleki2018maximum}). %
In the case where the source dynamics are \emph{not} equal to the true dynamics, this objective is not the same as maximum entropy RL on trajectories sampled from the source domain. Instead, this objective suggests a corrective term $\Delta r$ that should be added to the reward function to account for the discrepancy between the source and target dynamics. The correction term, $\Delta r$, is quite intuitive. If a transition $(s_t, a_t, s_{t+1})$ has equal probability in the source and target domains, then $\Delta r(s_t, a_t) = 0$ so no correction is applied. For transitions that are likely in the source but are unlikely in the target domain, $\Delta r < 0$, the agent is penalized for ``exploiting'' inaccuracies or discrepancies in the source domain by taking these transitions. For the example environment in Figure~\ref{fig:teaser}, transitions through the center of the environment are blocked in the target domain but not in the source domain. For these transitions, $\Delta r$ would serve as a large penalty, discouraging the agent from taking these transitions and instead learning to navigate around the wall. Appendix~\ref{appendix:perspectives} presents additional interpretations of $\Delta r$ in terms of coding theory, mutual information, and a constraint on the discrepancy between the source and target dynamics. Appendix~\ref{appendix:observation-model} discusses how prior work on domain agnostic feature learning can be viewed as a special case of our framework.

\subsection{Theoretical Guarantees}
\label{sec:guarantees}

We now analyze when maximizing the modified reward $r + \Delta r$ in the source domain yields a near-optimal policy for the target domain. Our proof relies on the following lightweight assumption:

\begin{assumption} \label{assumption:opt}
Let $\pi^* = \argmax_\pi \E_p\left[ \sum r(s_t, a_t) \right]$ be the reward-maximizing policy in the target domain. Then the expected reward in the source and target domains differs by at most $2 R_\text{max} \sqrt{\epsilon / 2}$:
\begin{equation*}
 \left| \E_{p_{\pi^*, \text{source}}}\left[ \sum r(s_t, a_t) \right] - \E_{\pi^*, p_{\text{target}}}\left[ \sum r(s_t, a_t) \right] \right| \le 2 R_\text{max} \sqrt{\epsilon / 2}.
\end{equation*}
\end{assumption}
The variable $R_\text{max}$ refers to the maximum entropy-regularized return of any trajectory. This assumption says that the optimal policy in the target domain is still a good policy for the source domain, and its expected reward is similar in both domains. We do not require that the opposite be true: the optimal policy in the source domain does not need to receive high reward in the target domain.
If there are multiple optimal policies, we only require that this assumption hold for one of them.
We now state our main result:

\begin{theorem} \label{thm:main}
Let $\pi_\text{DARC}^*$ be the policy that maximizes the modified (entropy-regularized) reward in the source domain, let $\pi^*$ be the policy that maximizes the (unmodified, entropy-regularized) reward in the target domain, and assume that $\pi^*$ satisfies Assumption~\ref{assumption:opt}.
Then the following holds:
\begin{equation*}
    \E_{p_\text{target}, \pi_\text{DARC}^*} \left[ \sum r(s_t, a_t) + \gH[a_t \mid s_t] \right] \ge \E_{p_\text{target}, \pi^*} \left[ \sum r(s_t, a_t) + \gH[a_t \mid s_t] \right] - 4 R_\text{max}\sqrt{\epsilon / 2}.
\end{equation*}
\end{theorem}
See Appendix~\ref{appendix:proofs} for the proof and definition of $\epsilon$. This result says that $\pi_\text{DARC}^*$ attains near-optimal (entropy-regularized) reward on the target domain. Thus, we can expect that modifying the reward function should allow us to adapt to different dynamics. The next section will present a practical algorithm for acquiring $\pi_\text{DARC}^*$ by estimating and effectively maximizing the modified reward in the source domain.

\section{Domain Adaptation in RL with a Learned Reward}
\label{sec:method}

\begin{wrapfigure}[12]{r}{0.5\textwidth}
    \centering
    \includegraphics[width=\linewidth]{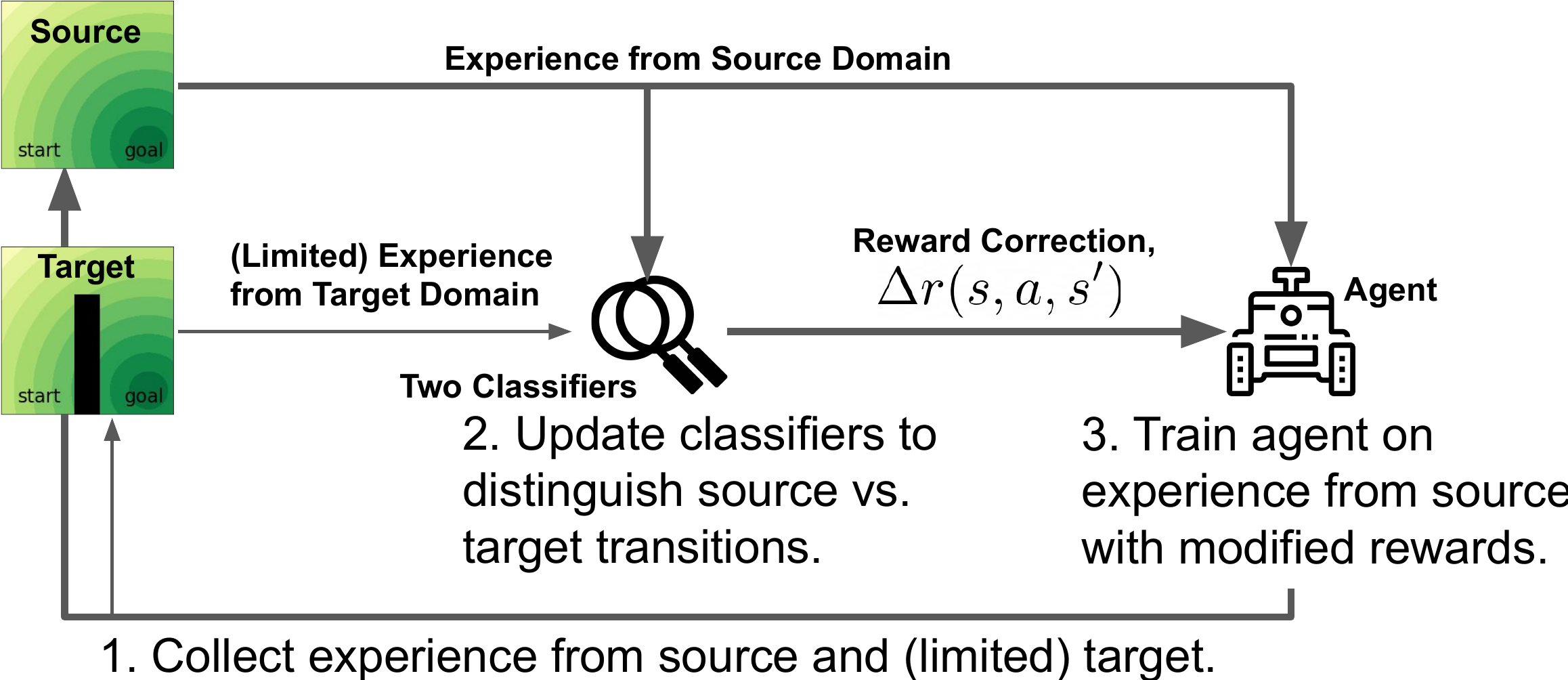}
    \vspace{-1em}
    \caption{Block diagram of DARC (Alg.~\ref{alg:odrl})}
    \label{fig:alg-figure}
\end{wrapfigure}
The variational perspective on model-based RL in the previous section suggests that we should modify the reward in the source domain by adding $\Delta r$.
In this section we develop a practical algorithm for off-dynamics RL by showing how $\Delta r$ can be estimated without learning an explicit dynamics model. %

\begin{algorithm}[t]
\caption{Domain Adaptation with Rewards from Classifiers [DARC]}\label{alg:odrl}
\begin{algorithmic}[1]
\State \textbf{Input:} source MDP $\gM_{\text{source}}$ and target $\gM_{\text{target}}$; ratio $r$ of experience from source vs. target.
\State \textbf{Initialize:} replay buffers for source and target transitions, $\gD_{\text{source}}, \gD_{\text{target}}$; policy $\pi$; parameters $\theta = (\theta_{\text{SAS}}, \theta_{\text{SA}})$ for classifiers $q_{\theta_{\text{SAS}}}(\text{target} \mid s_t, a_t, s_{t+1})$ and $q_{\theta_{\text{SAS}}}(\text{target} \mid s_t, a_t)$. 
\For{$t = 1, \cdots, \text{num iterations}$}
\State $\gD_\text{source} \gets \gD_\text{source} \cup \Call{Rollout}{\pi, \gM_{\text{source}}}$ \Comment{Collect source data.}
\If{$t \mod r = 0$} \Comment{Periodically, collect target data.} \label{line:mod-r}
\State $\gD_\text{target} \gets \gD_\text{target} \cup \Call{Rollout}{\pi, \gM_{\text{target}}}$
\EndIf

\State $\theta \gets \theta - \eta \nabla_\theta \ell(\theta)$ \Comment{Update both classifiers.}
\State $\tilde{r}(s_t, a_t, s_{t+1}) \gets r(s_t, a_t) + \Delta r(s_t, a_t, s_{t+1})$ \Comment{$\Delta r$ is computed with Eq.~\ref{eq:delta-r}.}
\State $\pi \gets \Call{MaxEnt RL}{\pi, \gD_\text{source}, \tilde{r}}$
\EndFor
\State \textbf{return} $\pi$
\end{algorithmic}
\end{algorithm}

To estimate $\Delta r$, we will use a pair of (learned) binary classifiers, which will infer whether transitions came from the source or target domain. The key idea is that the transition probabilities are related to the classifier probabilities via Bayes' rule:
\begin{equation*}
    p(\text{target} \mid s_t, a_t, s_{t+1}) = \underbrace{p(s_{t+1} \mid s_t, a_t, \text{target})}_{=p_\text{target}(s_{t+1} \mid s_t, a_t)} p(s_t, a_t \mid \text{target}) p(\text{target}) / p(s_t, a_t, s_{t+1}).
\end{equation*}
We estimate the term $p(s_t, a_t \mid \text{target})$ on the RHS via \emph{another} classifier, $p(\text{target} \mid s_t, a_t)$:
\begin{equation*}
    p(s_t, a_t \mid \text{target}) = \frac{p(\text{target} \mid s_t, a_t) p(s_t, a_t)}{p(\text{target})}. \label{eq:sa-classifier-bayes}
\end{equation*}
Substituting these expression into our definition for $\Delta r$ and simplifying, we obtain an estimate for $\Delta r$ that depends solely on the predictions of these two classifiers:
\begin{align}
    \Delta r(s_t, a_t, s_{t+1}) &= {\color{orange}\dotuline{\log p(\text{target} \mid s_t, a_t, s_{t+1})}} - {\color{blue}\dashuline{\log p(\text{target} \mid s_t, a_t)}} \nonumber \\
    & - {\color{orange}\dotuline{\log p(\text{source} \mid s_t, a_t, s_{t+1})}}  + {\color{blue}\dashuline{\log p(\text{source} \mid s_t, a_t)}} \label{eq:delta-r}
\end{align}
The {\color{orange}\dotuline{orange}} terms are the difference in logits from the classifier conditioned on $s_t, a_t, s_{t+1}$, while the {\color{blue}\dashuline{blue}} terms are the difference in logits from the classifier conditioned on just $s_t, a_t$. Intuitively, $\Delta r$ answers the following question: for the task of predicting whether a transition came from the source or target domain, how much better can you perform after observing $s_{t+1}$?
We make this connection precise in Appendix~\ref{appendix:mi} by relating $\Delta r$ to mutual information. Ablation experiments (Fig.~\ref{fig:ablation}) confirm that both classifiers are important to the success of our method. The use of transition classifiers makes our method look somewhat similar to adversarial imitation learning~\citep{ho2016generative, fu2017learning}. While our method is \emph{not} solving an imitation learning problem (we do not assume access to any expert experience), our method can be interpreted as learning a \emph{policy} such that the dynamics observed by that policy in the source domain imitate the dynamics of the target~domain.

\paragraph{Algorithm Summary}
Our algorithm modifies an existing MaxEnt RL algorithm to additionally learn two classifiers, $q_{\theta_{\text{SAS}}}(\text{target} \mid s_t, a_t, s_{t+1})$ and $q_{\theta_{\text{SA}}}(\text{target} \mid s_t, a_t)$, parametrized by $\theta_{\text{SAS}}$ and $\theta_{\text{SA}}$ respectively, to minimize the standard cross-entropy loss.
\begin{align*}
    \ell_{\text{SAS}}(\theta_\text{SAS}) & \triangleq -\E_{\gD_\text{target}} \left[ \log q_{\theta_\text{SAS}}(\text{target} \mid s_t, a_t, s_{t+1}) \right] - \E_{\gD_\text{source}} \left[\log q_{\theta_\text{SA}}(\text{source} \mid s_t, a_t, s_{t+1}) \right] \\
    \ell_{\text{SA}}(\theta_\text{SA}) & \triangleq - \E_{\gD_\text{target}} \left[ \log q_{\theta_\text{SA}}(\text{target} \mid s_t, a_t) \right] - \E_{\gD_\text{source}} \left[\log q_{\theta_\text{SA}}(\text{source} \mid s_t, a_t) \right].
\end{align*}

Our algorithm, Domain Adaptation with Rewards from Classifiers (DARC), is presented in Alg.~\ref{alg:odrl} and illustrated in Fig.~\ref{fig:alg-figure}.
To simplify notation, we define $\theta \triangleq (\theta_{\text{SAS}}, \theta_{\text{SA}})$ and $\ell(\theta) \triangleq \ell_{\text{SAS}}(\theta_{\text{SAS}}) + \ell_{\text{SA}}(\theta_{\text{SA}})$.
At each iteration, we collect transitions from the source and (less frequently) target domain, storing the transitions in separate replay buffers. We then sample a batch of experience from both
buffers to update the classifiers. We use the classifiers to modify the rewards from the \emph{source} domain, and apply MaxEnt RL to this experience. We use SAC~\citep{haarnoja2018soft} as our MaxEnt RL algorithm, but emphasize that DARC is applicable to any MaxEnt RL algorithm (e.g., on-policy, off-policy, and model-based). When training the classifiers, we 
add Gaussian input noise to prevent overfitting to the small number of target-domain transitions (see Fig.~\ref{fig:ablation} for an ablation).
Code has been released: {\scriptsize \url{https://github.com/google-research/google-research/tree/master/darc}}

\vspace{-0.5em}
\section{Experiments}
\vspace{-0.5em}

We start with a didactic experiment to build intuition for the mechanics of our method, and then evaluate on more complex tasks. Our experiments will show that DARC outperforms alternative approaches, such as directly applying RL to the target domain or learning importance weights. We will also show that our method can account for domain shift in the termination condition, and confirm the importance of learning two classifiers.

\begin{wrapfigure}[8]{r}{0.5\textwidth}
  \vspace{-1em}
  \includegraphics[width=0.5\textwidth]{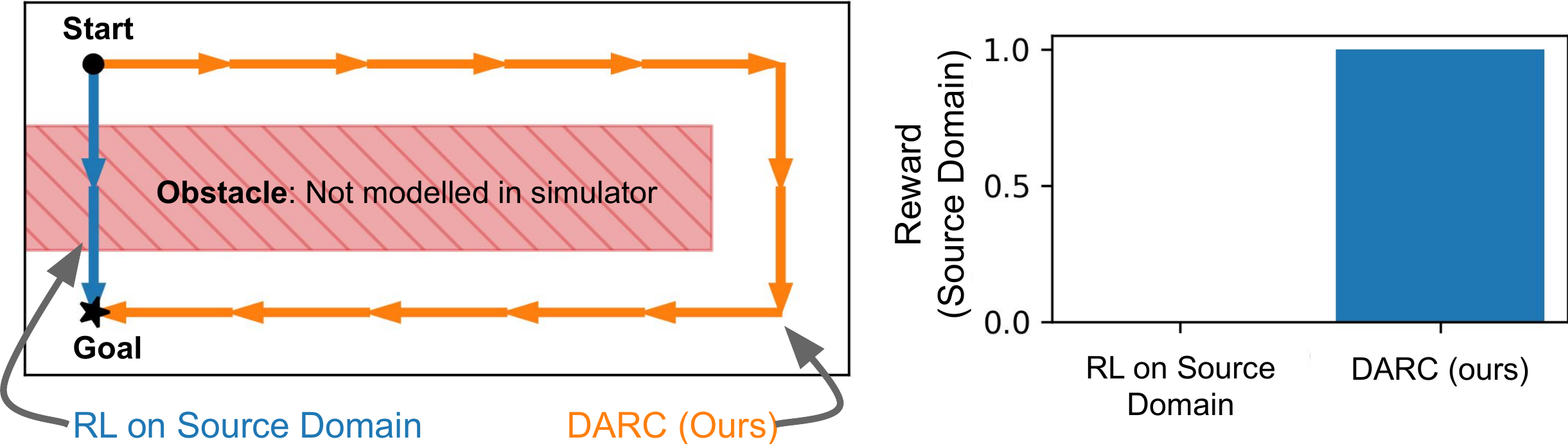}
   \vspace{-1.5em}
   \caption{\textbf{Tabular example of off-dynamics RL}}
\end{wrapfigure}
\paragraph{Illustrative example.}
We start with a simple gridworld example, shown on the right, where we can apply our method without function approximation. The goal is to navigate from the top left to the bottom left. The real environment contains an obstacle (shown in red), which is not present in the source domain. If we simply apply RL on the source domain, we obtain a policy that navigates directly to the goal (blue arrows), and will fail when used in the target domain. We then apply our method: we collect trajectories from the source domain and real world to fit the two tabular classifiers. These classifiers give us a modified reward, which we use to learn a policy in the source domain. The modified reward causes our learned policy to navigate around the obstacle, which succeeds in the target environment.

\begin{wrapfigure}[9]{r}{0.5\textwidth}
  \vspace{-1.5em}
  \includegraphics[width=\linewidth]{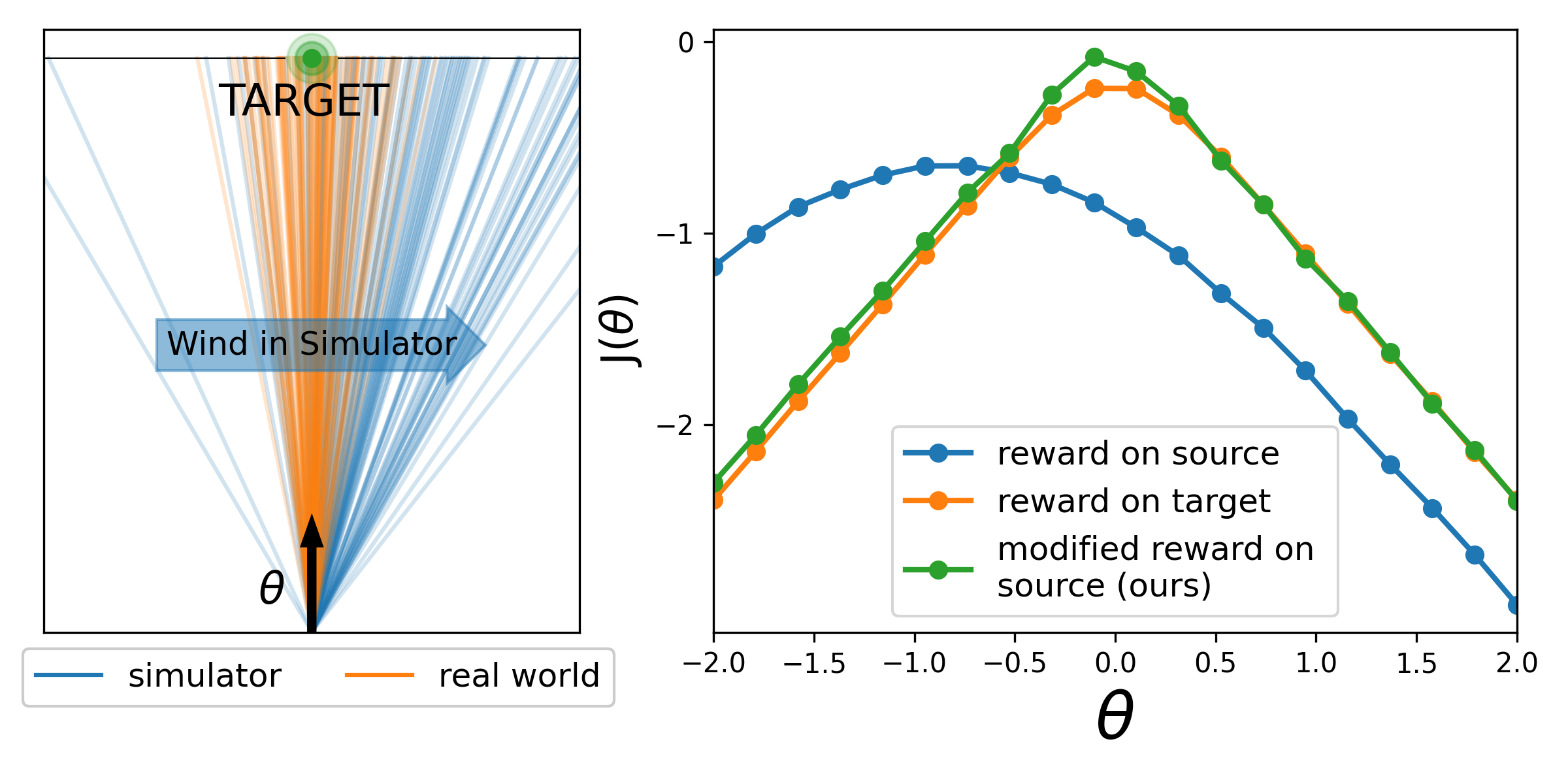}
  \vspace{-2.5em}
   \caption{\textbf{Visualizing the modified reward} \label{fig:arrow}}
\end{wrapfigure}
\paragraph{Visualizing the reward modification in stochastic domains.} In our next experiment, we use an ``archery'' task to visualize how the modified reward accounts for differences in dynamics. The task, shown in Fig.~\ref{fig:arrow}, requires choosing an angle at which to shoot an arrow. The practice range (i.e., the source domain) is outdoors, with wind that usually blows from left to right. The competition range (i.e., the target domain) is indoors with no wind. The reward is the negative distance to the target. We plot the reward as a function of the angle in both domains in Fig.~\ref{fig:arrow}. The optimal strategy for the outdoor range is to compensate for the wind by shooting slightly to the left ($\theta=-0.8$), while the optimal strategy for the indoor range is to shoot straight ahead ($\theta = 0$). We estimate the modified reward function with DARC, and plot the modified reward in the windy outdoor range and indoor range. We aggregate across episodes using $J(\theta) = \log \E_{p(s' \mid \theta)}[\exp(r(s'))]$; see Appendix~\ref{appendix:archery} for details. We observe that maximizing the modified reward in the windy range does not yield high reward in the windy range, but does yield a policy that performs well in the indoor range.

\begin{wrapfigure}[7]{r}{0.5\textwidth}
\vspace{-0.8em}
    \centering
    \begin{subfigure}[t]{0.23\linewidth}
        \centering
        \includegraphics[width=\linewidth]{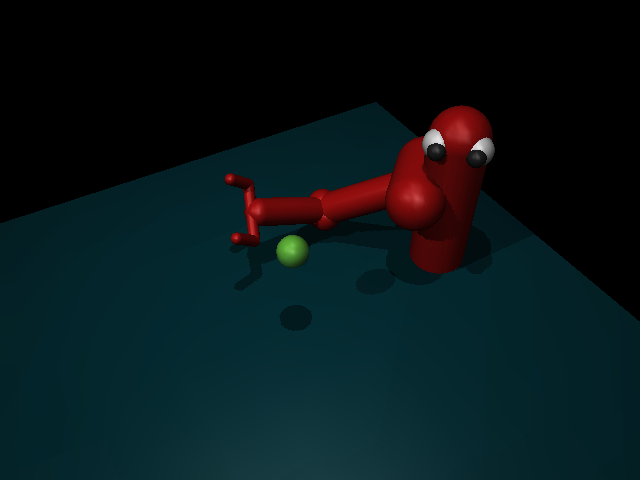}
    \end{subfigure}%
    ~ 
    \begin{subfigure}[t]{0.23\linewidth}
        \centering
        \includegraphics[width=\linewidth]{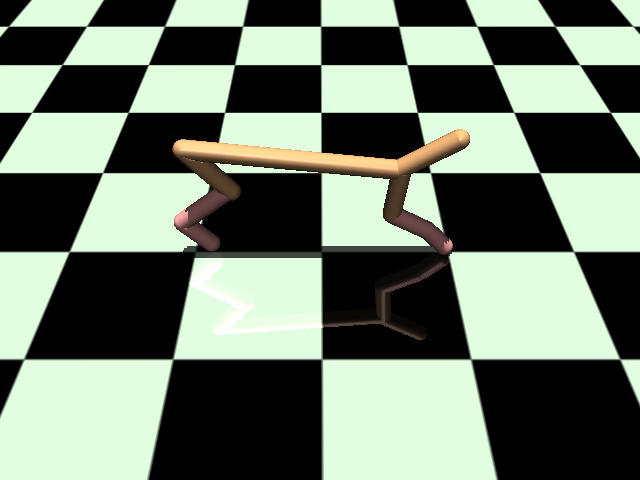}
    \end{subfigure}%
    ~ 
    \begin{subfigure}[t]{0.23\linewidth}
        \centering
        \includegraphics[width=\linewidth]{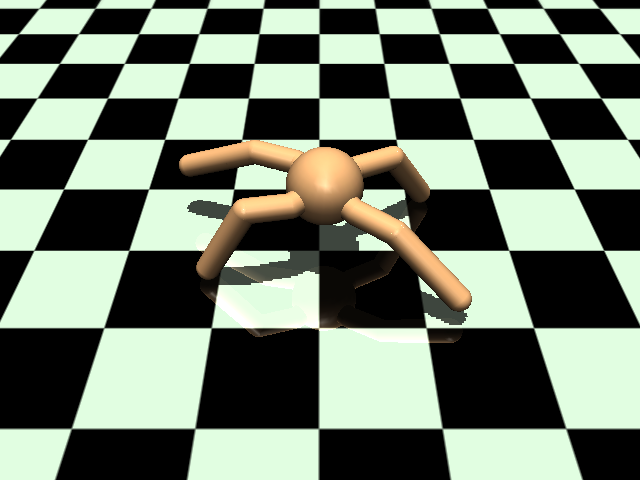}
    \end{subfigure}%
    ~
    \begin{subfigure}[t]{0.23\linewidth}
        \centering
        \includegraphics[width=\linewidth]{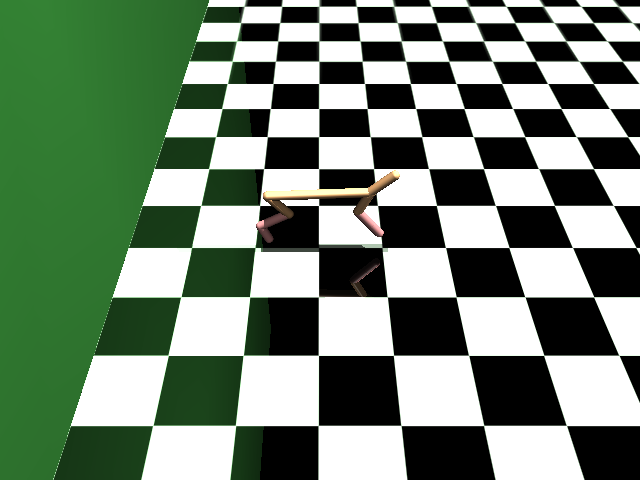}
    \end{subfigure}
    \caption{{\footnotesize \textbf{Environments}: broken reacher, broken half cheetah, broken ant, and half cheetah obstacle.}}
    \label{fig:environments}
\vspace{-0.15in}
\end{wrapfigure}
\paragraph{Scaling to more complex tasks.}
We now apply DARC to the more complex tasks shown in Fig.~\ref{fig:environments}. We define three tasks by crippling one of the joints of each robot in the target domain, but using the fully-functional robot in the source domain. We use three simulated robots taken from
OpenAI Gym~\citep{brockman2016openai}: 7 DOF reacher, half cheetah, and ant. The broken reacher is based on the task described by \citet{vemula2020planning}. We also include a task where the shift in dynamics is external to the robot, by modifying the cheetah task to reward the agent for running both forward and backwards. It is easier to learn to run backwards, an obstacle in the target domain prevents the agent from running~backwards. This ``half cheetah obstacle'' task does not entirely satisfy the assumption in Eq.~\ref{eq:completeness} because transitions such as bouncing off the obstacle are only possible in the target domain, not the source domain. Nonetheless, the success of our method on this task illustrates that DARC can excel even in settings that do not satisfy the assumption in Eq.~\ref{eq:completeness}.

We compare our method to eight baselines. \textbf{RL on Source} and \textbf{RL on Target} directly perform RL on the source and target domains, respectively. The \textbf{Finetuning} baseline takes the result of running RL on the source domain, and further finetunes the agent on the target domain. The \textbf{Importance Weighting} baseline performs RL on importance-weighted samples from the source domain; the importance weights are $\exp(\Delta r)$. Recall that DARC collects many ($r = 10$) transitions in the source domain and performs many gradient updates for each single transition collected in the target domain (Alg.~\ref{alg:odrl} Line~\ref{line:mod-r}). We therefore compared against a \textbf{RL on Target (10x)} baseline that likewise performs many ($r = 10$) gradient updates per transition in the target domain. Next, we compared against two recent model-based RL methods: \textbf{MBPO}~\citep{janner2019trust} and \textbf{PETS}~\citep{chua2018deep}. Finally, we also compared against \textbf{MATL}~\citep{wulfmeier2017mutual}, which is similar in spirit to our method but uses a different modified reward.

\begin{figure}[t]
    \vspace{-1em}
    \centering
     \includegraphics[width=\linewidth]{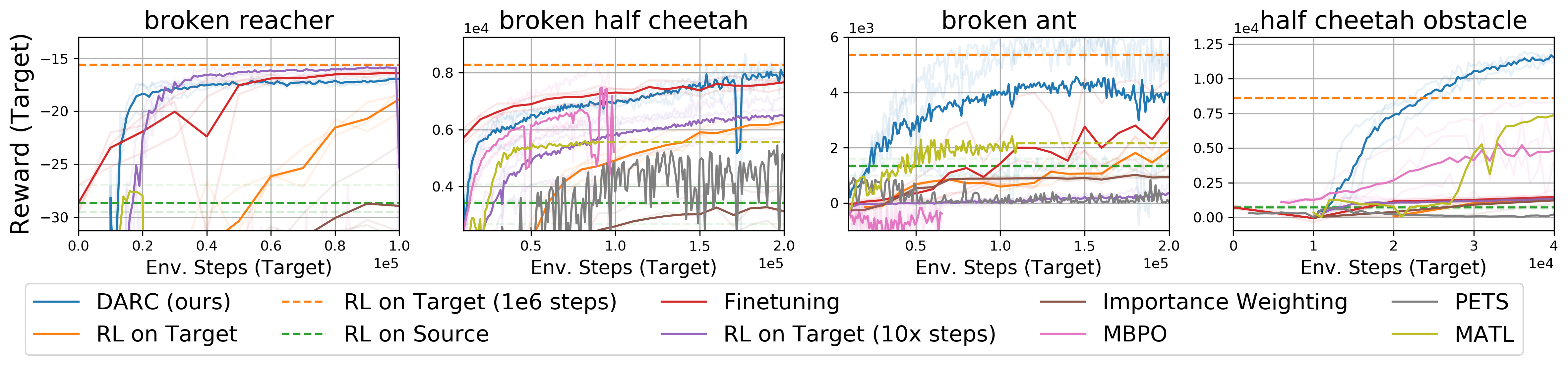}
    \vspace{-1.5em}
    \caption{{\footnotesize \textbf{DARC compensates for crippled robots and obstacles}: We apply DARC to four continuous control tasks: three tasks (broken reacher, half cheetah, and ant) which are crippled in the target domain but not the source domain, and one task (half cheetah obstacle) where the source domain omits the obstacle from the target domain.  Note that na\"ively ignoring the shift in dynamics (green dashed line) performs quite poorly, while directly learning on the crippled robot requires an order of magnitude more experience than our method.}}
    \label{fig:broken-robots}
    \vspace{-1em}
\end{figure}

We show the results of this experiment in Fig.~\ref{fig:broken-robots}, plotting the reward on the \emph{target} domain as a function of the number of transitions in the \emph{target} domain. In this figure, the transparent lines correspond to different random seeds, and the darker lines are the average of these random seeds. On all tasks, the RL on source baseline (shown as a dashed line because it observes no target transitions) performs considerably worse than the optimal policy from RL on the target domain, suggesting that good policies for the source domain are suboptimal for the target domain. Nonetheless, on three of the four tasks our method matches (or even surpasses) the asymptotic performance of doing RL on the target domain, despite never doing RL on experience from the target domain, and despite observing 5 -- 10$\times$ less experience from the target domain. On the broken reacher and broken half cheetah tasks, finetuning on the target domain performs on par with our method.
On the simpler broken reacher task, just doing RL on the target domain with a large number of gradient steps works quite well (we did not tune this parameter for our method). 
While the model-based baselines (PETS and MBPO) also performed well on for low-dimensional tasks (broken reacher, broken half cheetah), they perform quite poorly on more challenging tasks like broken ant, supporting our intuition that classification is easier than learning a dynamics model in high dimensional tasks.
Finally, DARC outperforms MATL on all tasks. %

\begin{figure}[t]
    \centering
    \vspace{-1.7em}
    \begin{subfigure}[t]{0.49\linewidth}
        \centering
        \includegraphics[width=\linewidth]{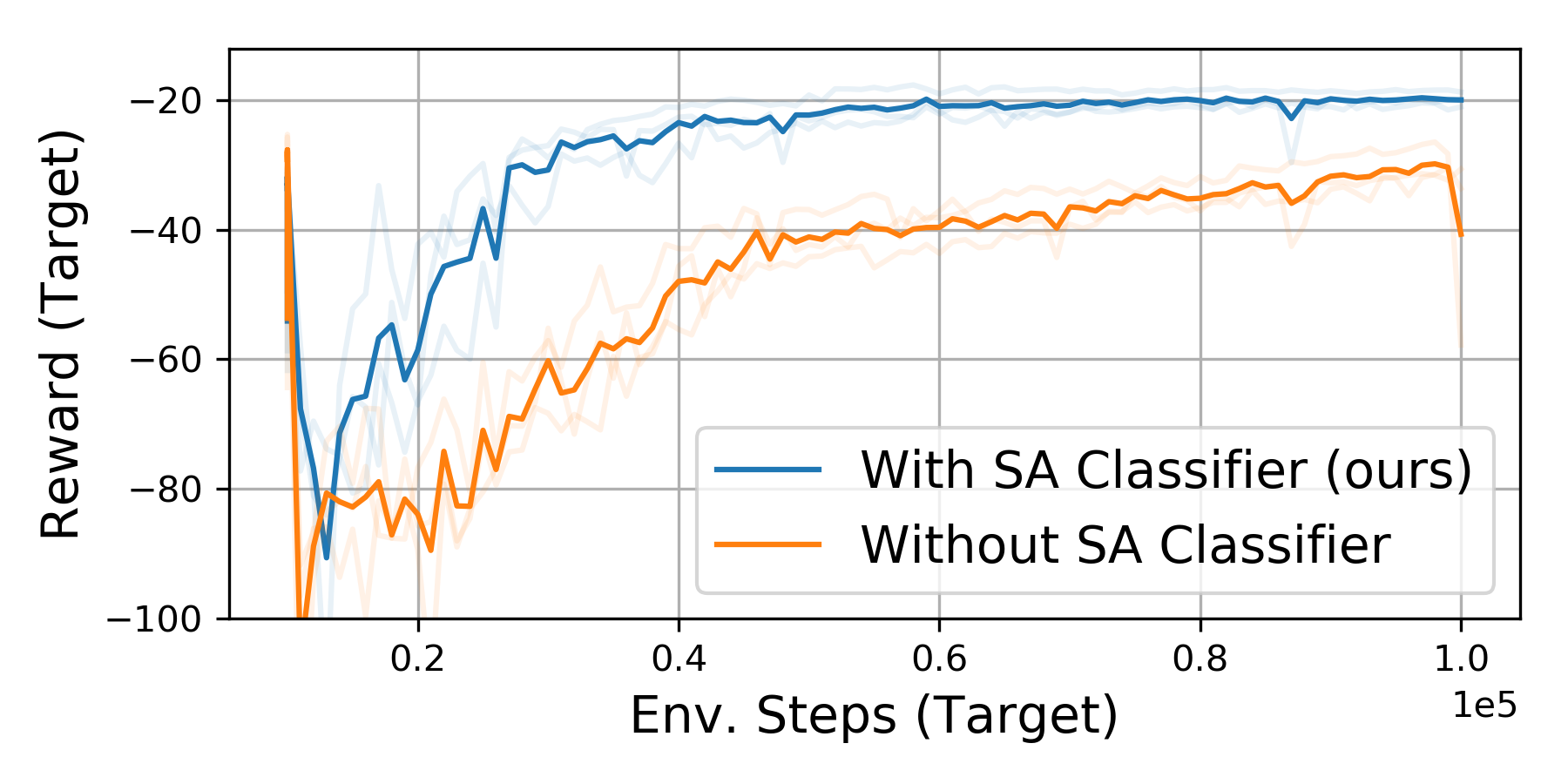}
    \end{subfigure}%
    ~ 
    \begin{subfigure}[t]{0.49\linewidth}
        \centering
        \includegraphics[width=\linewidth]{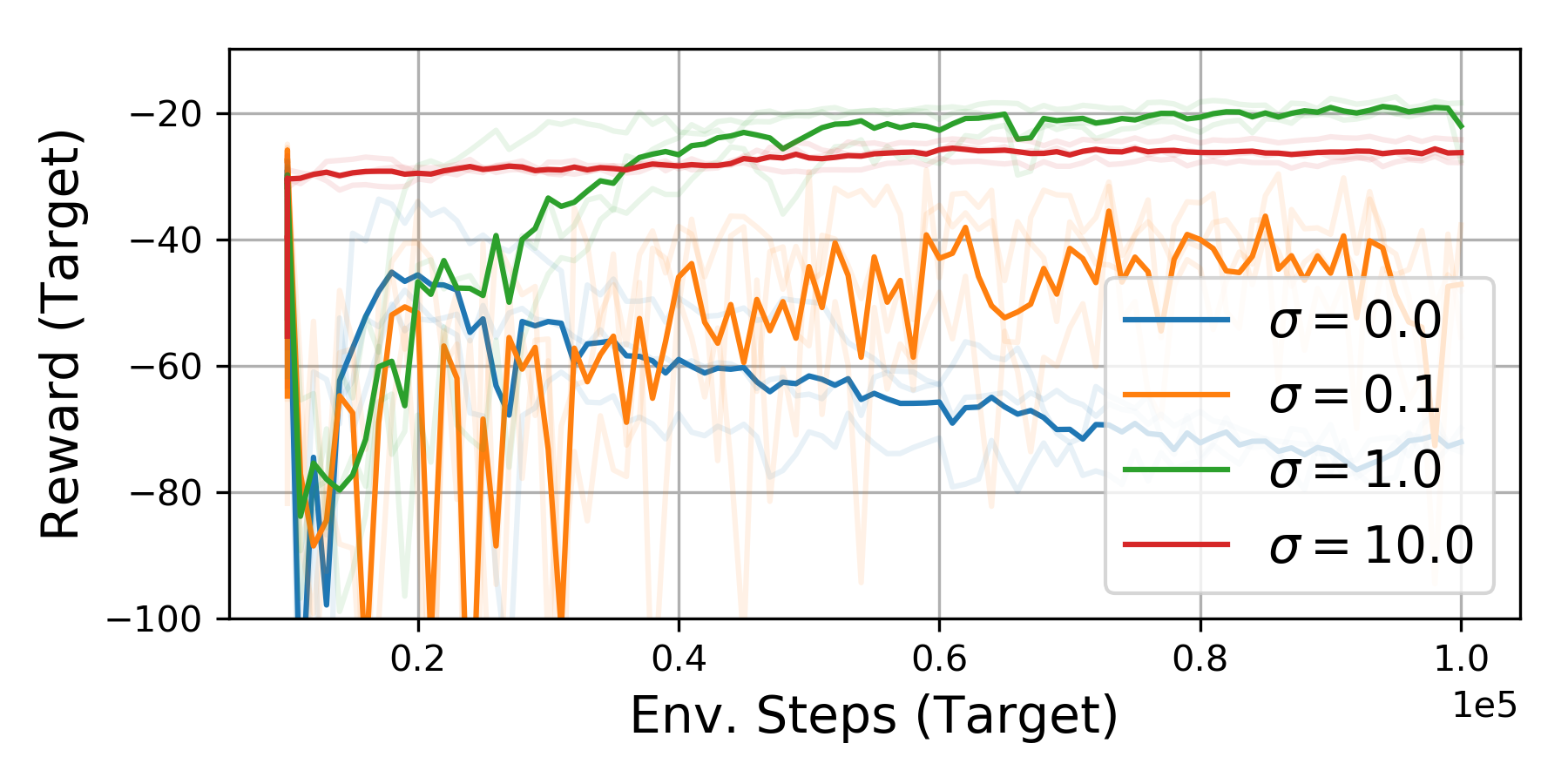}
    \end{subfigure}
    \vspace{-1em}
    \caption{{\footnotesize \textbf{Ablation experiments} \figleft\; DARC performs worse when only one classifier is used. \figright \; Using input noise to regularize the classifiers boosts performance. Both plots show results for broken reacher; see Appendix~\ref{appendix:ablation} for results on all environments.}}
    \label{fig:ablation}
    \vspace{-1.5em}
\end{figure}

\paragraph{Ablation Experiments.}
Our next experiment examines the importance of using two classifiers to estimate $\Delta r$.
We compared our method to an ablation that does not learn the SA classifier, effectively ignoring the {\color{blue} \dashuline{blue}} terms in Eq.~\ref{eq:delta-r}. As shown in Fig.~\ref{fig:ablation} (left), this ablation performs considerably worse than our method. Intuitively, this makes sense: we might predict that a transition came from the source domain not because the next state had higher likelihood under the source dynamics, but rather because the state or action was visited more frequently in the source domain. The second classifier used in our method corrects for this distribution shift.

Next, we examine the importance of input noise regularization in classifiers. As we observe only a handful of transitions from the target domain, we hypothesized that regularization would be important to prevent overfitting. We test this hypothesis in Fig.~\ref{fig:ablation} (right) by training our method on the broken reacher environment with varying amounts of input noise. With no noise or little noise our method performs poorly (likely due to overfitting); too much noise also performs poorly (likely due to underfitting). We used a value of 1 in all our experiments, and did not tune this value. See Appendix~\ref{appendix:ablation} for more plots of both ablation experiments.

\begin{wrapfigure}[17]{r}{0.5\textwidth}
    \centering
    \vspace{-1.em}
    \includegraphics[width=\linewidth]{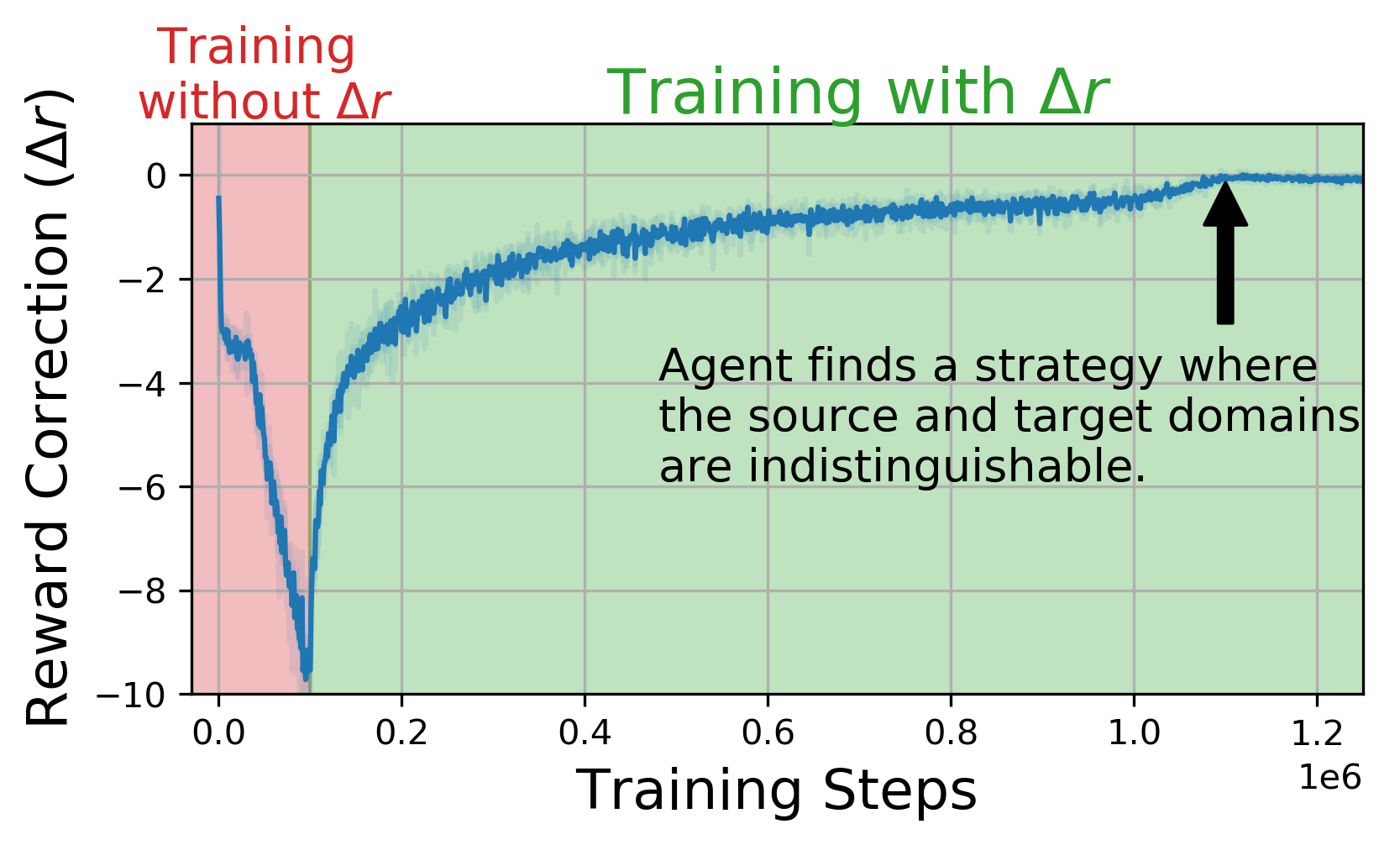}
    \vspace{-2em}
    \caption{{\footnotesize Without the reward correction, the agent takes transitions where the source domain and target domains are dissimilar; after adding the reward correction, the agent's transitions in the source domain are increasingly likely under the target domain. %
    }}
    \label{fig:delta-r}
\end{wrapfigure}
To gain more intuition for our method, we recorded the reward correction $\Delta r$ throughout training on the broken reacher environment. In this experiment, we ran RL on the source domain for 100k steps before switching to our method. Said another way, we ignored $\Delta r$ for the first 100k steps of training. As shown in Fig.~\ref{fig:delta-r}, $\Delta r$ steadily decreases during these first 100k steps, suggesting that the agent is learning a strategy that takes transitions where the source domain and target domain have different dynamics: the agent is making use of its broken joint. After 100k steps, when we maximize the combination of task reward and  $\Delta r$, we observe that $\Delta r$ increases, so the agent's transitions in the source domain are increasingly consistent with target domain dynamics. After around 1e6 training steps $\Delta r$ is zero: the agent has learned a strategy that uses transitions that are indistinguishable between the source and target domains.

\paragraph{Safety emerges from domain adaptation to the termination condition.}

In many safety-critical applications, the real-world and simulator have different safeguards, which kick in to stop the agent and terminate the episode. For an agent to effectively transfer from the simulator to the real world, it cannot rely on safeguards which are present in one domain but not the other. Since this termination condition is part of the dynamics~\citep{white2017unifying}, we can readily apply DARC to this setting. %

\begin{wrapfigure}[12]{r}{0.5\textwidth}
    \centering
    \vspace{-0.5em}
    \begin{subfigure}[c]{0.2\linewidth}
        \centering
        \includegraphics[width=\linewidth]{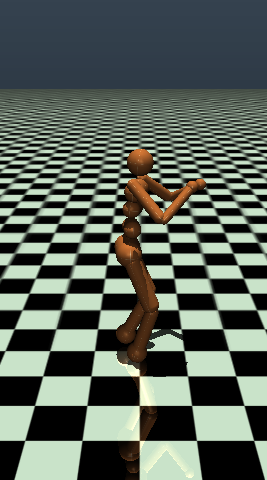}
    \end{subfigure}%
    ~ 
    \begin{subfigure}[c]{0.8\linewidth}
        \centering
        \includegraphics[width=\linewidth]{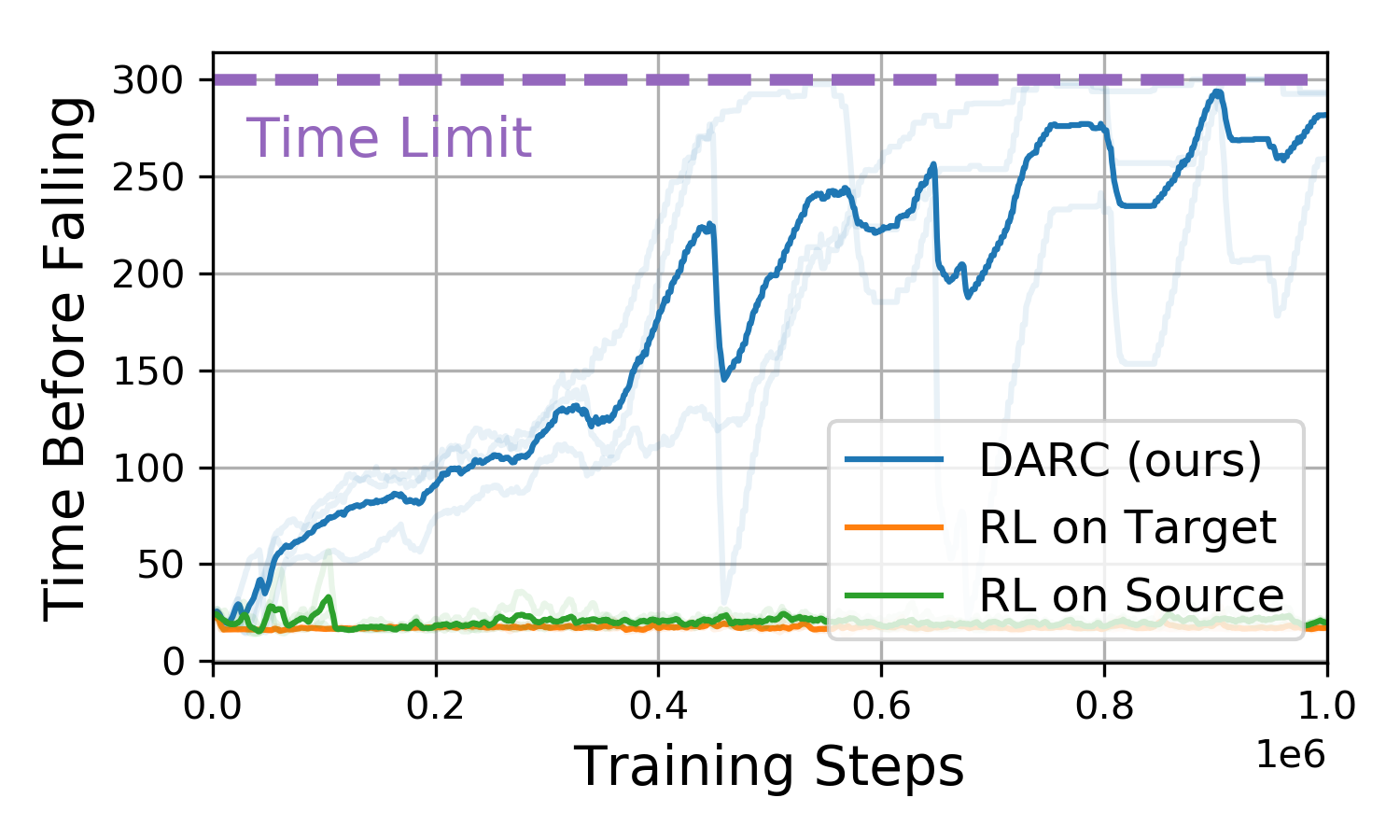}
    \end{subfigure}%
    \vspace{-1em}
    \caption{{\footnotesize Our method accounts for domain shift in the termination condition, causing the agent to avoid transitions that cause termination in the target domain.}}
    \label{fig:falling}
\end{wrapfigure}
We use the humanoid shown in Fig.~\ref{fig:falling} for this experiment and set the task reward to 0. In the source domain episodes have a fixed length of 300 steps; in the target domain the episode terminates when the robot falls. The scenario mimics the real-world setting where robots have freedom to practice in a safe, cushioned, practice facility, but are preemptively stopped when they try to take unsafe actions in the real world. Our aim is for the agent to learn to avoid unsafe transitions in the source domain that would result in episode termination in the target domain. %
As shown in Fig.~\ref{fig:falling}, our method learns to remain standing for nearly the entire episode. As expected, baselines that maximize the zero reward on the source and target domains fall immediately. While DARC was not designed as a method for safe RL~\citep{tamar2013scaling, achiam2017constrained, eysenbach2017leave, berkenkamp2017safe}, this experiment suggests that safety may emerge automatically from DARC, without any manual reward function design.

\vspace{-0.5em}
\paragraph{Comparison with Prior Transfer Learning Methods.}

\begin{wrapfigure}[12]{r}{0.5\textwidth}
\vspace{-1.5em}
\includegraphics[height=3.9cm]{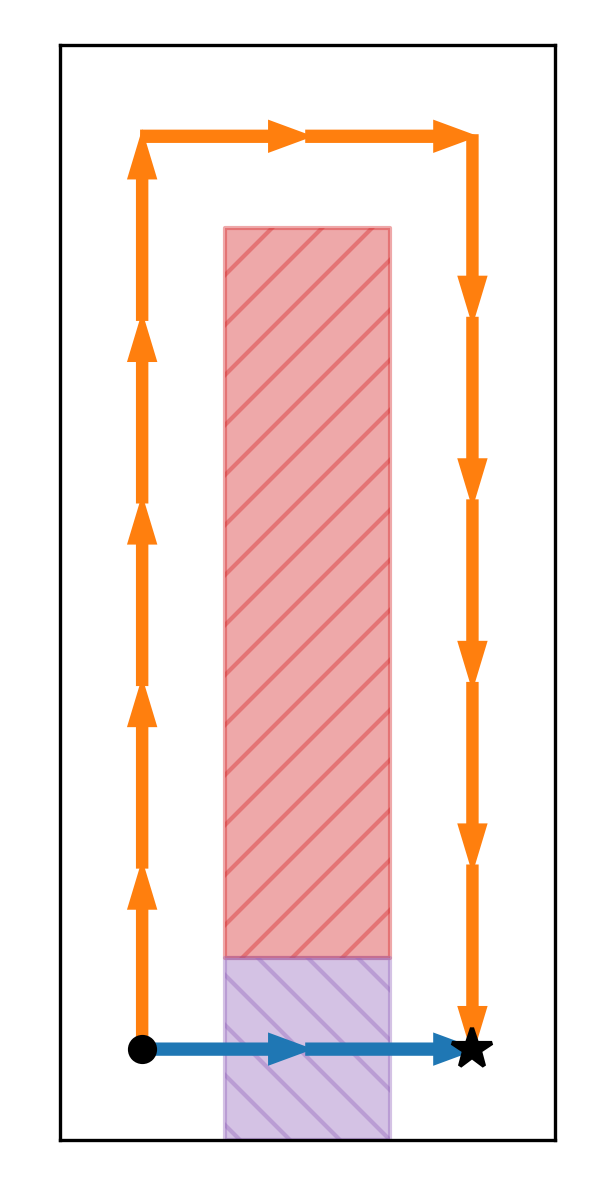}
\includegraphics[height=3.9cm]{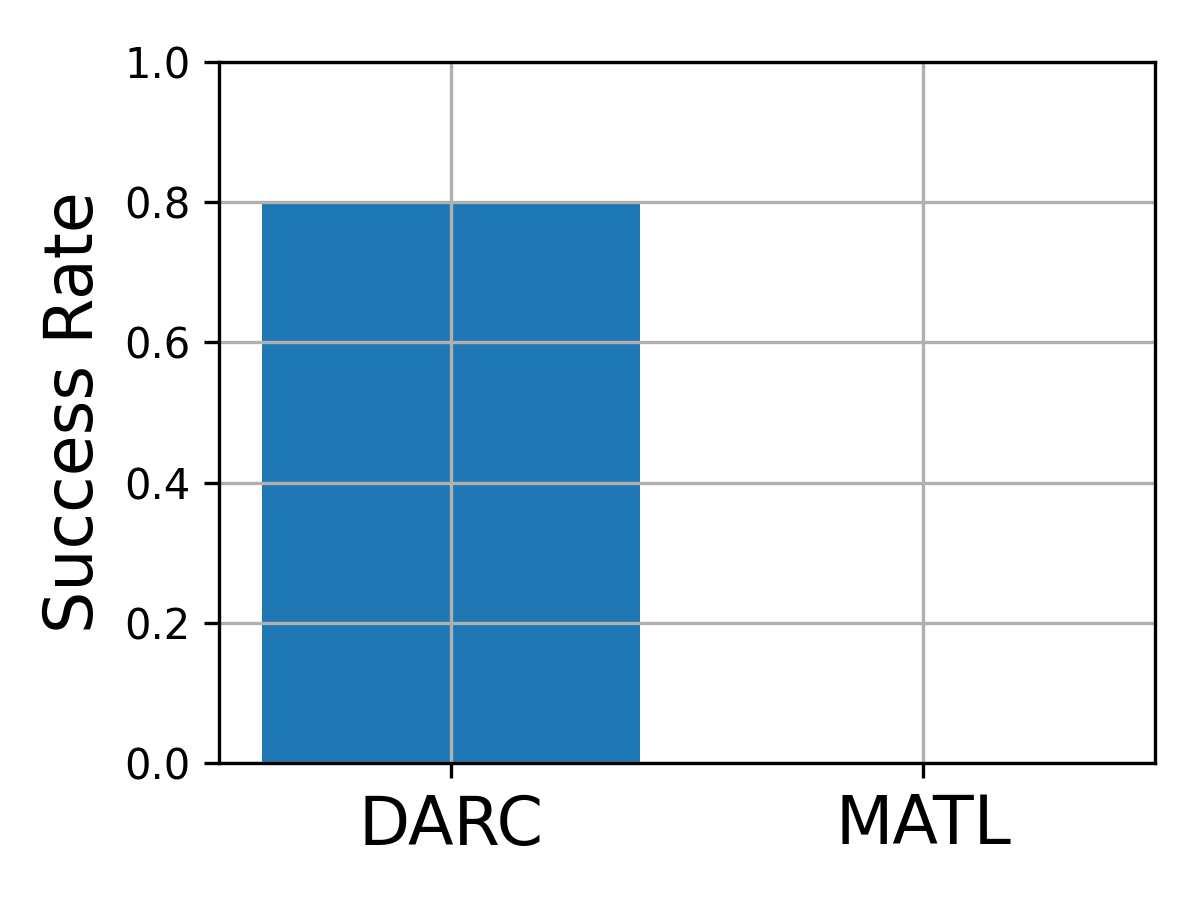}
\vspace{-1.7em}
\caption{Comparison with MATL} \label{fig:matl}
\end{wrapfigure}
We are not the first work that modifies the reward function to perform transfer in RL~\citep{koos2012transferability}, nor the first work to \emph{learn} how to modify the reward function~\citep{wulfmeier2017addressing}. However, these prior works lack theoretical justification. In contrast, our approach maximizes a well-defined variational objective and our analysis guarantees that agents learned with our method will achieve similar rewards in the source and target domains. Our formal guarantees (Sec.~\ref{sec:variational}) do not apply to MATL~\citep{wulfmeier2017mutual} because their classifier is not conditioned on the action.
Indeed, our results on the four tasks in Fig.~\ref{fig:broken-robots} indicate that DARC ourperforms MATL on all tasks. To highlight this difference, we compared DARC and MATL on a gridworld (right), where the source and target domains differed by assigning opposite effects to the ``up'' and ``down'' in the purple state in the source and target domains. We collected data from a uniform random policy, so the \emph{marginal} distribution $p(s_{t+1} \mid s_t)$ was the same in the source and target domains, even though the dynamics $p(s_{t+1} \mid s_t, a_t)$ where different. In this domain, MATL fails to recognize that the source and target domains are different. DARC succeeds in this task for 80\% of trials while MATL succeeds for 0\% of trials. 
We conclude that the conditioning on the action, as suggested by our analysis, is especially important when using experience collected from stochastic policies.

\vspace{-0.5em}
\section{Discussion}
\vspace{-0.5em}

In this paper, we proposed a simple, practical, and intuitive approach for domain adaptation to changing dynamics in RL. We motivate this method from a novel variational perspective on domain adaptation in RL, which suggests that we can compensate for differences in dynamics via the reward function. Moreover, we formally prove that, subject to a lightweight assumption, our method is guaranteed to yield a near-optimal policy for the target domain. Experiments on a range of control tasks show that our method can leverage the source domain to learn policies that will work well in the target domain, despite observing only a handful of transitions from the target domain.

\paragraph{Limitations} The main limitation of our method is that the source dynamics must be sufficiently stochastic, an assumption that can usually be satisfied by adding noise to the dynamics, or ensembling a collection of sources. Empirically, we found that our method worked best on tasks that could be completed in many ways in the source domain, but some of these strategies were not compatible with the target dynamics. The main takeaway of this work is that inaccuracies in dynamics can be compensated for via the reward function. In future work we aim to use the variation perspective on domain adaptation (Sec.~\ref{sec:variational}) to learn the dynamics for the source~domain.

\vspace{2em}
{\footnotesize
\paragraph{Acknowledgements.}
We thank Anirudh Vemula for early discussions; we thank Karol Hausman, Vincent Vanhoucke and anonymous reviews for feedback on drafts of this work. We thank Barry Moore for providing containers with MuJoCo and Dr. Paul Munro granting access to compute at CRC.
This work is supported by the Fannie and John Hertz Foundation, University of Pittsburgh Center for Research Computing (CRC), NSF (DGE1745016, IIS1763562), ONR (N000141812861), and US Army. Any opinions, findings and conclusions or recommendations expressed in this material are those of the author(s) and do not necessarily reflect the views of the National Science Foundation.

\paragraph{Contributions.}
BE proposed the idea of using rewards to correct for dynamics, designed and ran many of the experiments in the paper, and wrote much of the paper. SA did the initial literature review, wrote and designed some of the DARC experiments and environments, developed visualizations of the modified reward function, and ran the MBPO experiments. SC designed some of the initial environments, helped with the implementation of DARC, and ran the PETS experiments. RS and SL provided guidance throughout the project, and contributed to the structure and writing of the paper.
}

\clearpage
\appendix

\section{Additional Interpretations of the Reward Correction}
\label{appendix:perspectives}
This section presents four additional interpretations of the reward correction, $\Delta r$.

\subsection{Coding Theory}
\label{appendix:coding}
The reward correction $\Delta r$ can also be understood from the perspective of coding theory. Suppose that we use a data-efficient replay buffer that exploits that fact that the next state $s_{t+1}$ is highly redundant with the current state and action, $s_t, a_t$. If we assume that the replay buffer compression has been optimized to store transitions from the target environment,  (negative) $\Delta r$ is the number of additional bits (per transition) needed for our source replay buffer, as compared with our target replay buffer. Thus, an agent which maximizes $\Delta r$ will seek those transitions that can be encoded most efficiently, minimizing the size of the source replay buffer.

\subsection{Mutual Information}
\label{appendix:mi}
We can gain more intuition in the modified reward by writing the expected value of $\Delta r$ from Eq.~\ref{eq:delta-r} in terms of mutual information:
\begin{equation*}
    \E[\Delta r(s_t, a_t, s_{t+1})] = I(s_{t+1}; \text{target} \mid s_t, a_t) - I(s_{t+1}; \text{source} \mid s_t, a_t).
\end{equation*}
The mutual information $I(s_{t+1} ; \text{target} \mid s_t, a_t)$ reflects how much better you can predict the next state if you know that you are interacting with the target domain, instead of the source domain. Our approach does exactly this, rewarding the agent for taking transitions that provide information about the target domain while penalizing transitions that hint to the agent that it is interacting with a source domain rather than the target domain: we don't want our are agent to find bugs in the Matrix.%

\subsection{Lower bound on the risk-sensitive reward objective.}
\label{appendix:risk-sensitive}
While we derived DARC by minimizing a reverse KL divergence (Eq.~\ref{eq:kl}), we can also show that DARC maximizes a lower bound on a risk-sensitive reward objective~\citep{mihatsch2002risk}:
{\footnotesize
\begin{align}
    \log &\E_{\substack{s' \sim p_\text{target}(s' \mid s, a),\\a \sim \pi(a \mid s)}} \left[ \exp \left( \sum_t r(s_t, a_t)\right) \right] \nonumber \\
    &= \log \E_{\substack{s' \sim p_\text{source}(s' \mid s, a),\\a \sim \pi(a \mid s)}} \left[ \left( \prod_t\frac{p_\text{target}(s_{t+1} \mid s_t, a_t)}{p_\text{source}(s_{t+1} \mid s_t, a_t)} \right) \exp \left( \sum_t r(s_t, a_t)\right) \right] \nonumber \\
    &= \log \E_{\substack{s' \sim p_\text{source}(s' \mid s, a),\\a \sim \pi(a \mid s)}} \left[ \exp \left( \sum_t r(s_t, a_t) + \underbrace{\log p_\text{target}(s_{t+1} \mid s_t, a_t) - \log p_\text{source}(s_{t+1} \mid s_t, a_t)}_{\Delta r(s_t, a_t, s_{t+1})} \right) \right] \label{eq:before-jensen} \\
    & \ge  \E_{\substack{s' \sim p_\text{source}(s' \mid s, a),\\a \sim \pi(a \mid s)}} \left[\sum_t r(s_t, a_t) + \Delta r(s_t, a_t, s_{t+1}) \right]. \label{eq:after-jensen}
\end{align}
}
The inequality on the last line is an application of Jensen's inequality. One interesting question is when it would be preferable to maximize Eq.~\ref{eq:before-jensen} rather than Eq.~\ref{eq:after-jensen}. While Eq.~\ref{eq:after-jensen} provides a loser bound on the risk sensitive objective, empirically it may avoid the risk-seeking behavior that can be induced by risk-sensitive objectives. We leave the investigation of this trade-off as future work.

\subsection{A Constraint on Dynamics Discrepancy}
\label{appendix:dynamics-constraint}
Our method regularizes the policy to visit states where the transition dynamics are similar between the source domain and target domain:
\begin{equation*}
    \max_\pi \E_{\substack{a \sim \pi(a \mid s) \\ s' \sim p(s' \mid s, a)}} \left[\sum_t r(s_t, a_t) + \underbrace{\log p_\text{target}(s_{t+1} \mid s_t, a_t) - \log p_\text{source}(s_{t+1} \mid s_t, a_t)}_{-\kl{p_\text{source}}{p_\text{target}}} + \gH_\pi[a_t \mid s_t] \right].
\end{equation*}
This objective can equivalently be expressed as applying MaxEnt RL to only those policies which avoid exploiting the dynamics discrepancy. More precisely, the KKT conditions guarantee that there exists a positive constant $\epsilon > 0$ such that our objective is equivalent to the following constrained objective:
\begin{equation*}
    \max_{\pi \in \Pi_\text{DARC}} \E_{\substack{a \sim \pi(a \mid s) \\ s' \sim p(s' \mid s, a)}} \left[\sum_t r(s_t, a_t) + \gH_\pi[a_t \mid s_t] \right],
\end{equation*}
where $\Pi_\text{DARC}$ denotes the set of policies that do not exploit the dynamics discrepancy:
\begin{equation}
    \Pi_\text{DARC} \triangleq \left\{ \pi \Big \vert \E_{\substack{a \sim \pi(a \mid s) \\ s' \sim p(s' \mid s, a)}} \left[\sum_t \kl{p_\text{source}(s_{t+1} \mid s_t, a_t)}{p_\text{target}(s_{t+1} \mid s_t, a_t)} \right] \le \epsilon \right\}. \label{eq:constrained}
\end{equation}
One potential benefit of considering our method as the unconstrained objective is that it provides a principled method for increasing or decreasing the weight on the $\Delta r$ term, depending on how much the policy is currently exploiting the dynamics discrepancy. We leave this investigation as future~work.

\section{Proofs of Theoretical Guarantees}
\label{appendix:proofs}

In this section we present our analysis showing that maximizing the modified reward $r + \Delta r$ in the source domain yields a near-optimal policy for the target domain, subject to Assumption~\ref{assumption:opt}. To start, we show that maximizing the modified reward in the source domain is equivalent to maximizing the unmodified reward, subject to the constraint that the policy not exploit the dynamics:
\begin{lemma}
Let a reward function $r(s, a)$, source dynamics $p_\text{source}(s' \mid s, a)$, and target dynamics $p_\text{target}(s' \mid s, a)$ be given. Then there exists $\epsilon > 0$ such the optimization problem in Eq.~\ref{eq:kl} is equivalent to
\begin{equation*}
    \max_{\pi \in \Pi_{\text{no exploit}}} \E_{p_\text{source}, \pi} \left[\sum r(s_t, a_t) +  \gH_\pi[a_t \mid s_t] \right],
\end{equation*}
where $\Pi_\text{no exploit}$ denotes the set of policies that do not exploit the dynamics:
\begin{equation*}
    \Pi_\text{no exploit} \triangleq \left\{ \E_{\substack{a \sim \pi(a \mid s) \\ s' \sim p(s' \mid s, a)}} \left[\sum_t \kl{p_\text{source}(s_{t+1} \mid s_t, a_t)}{p_\text{target}(s_{t+1} \mid s_t, a_t)} \right] \le \epsilon \right \} .
\end{equation*}
\end{lemma}
The proof 
is a straightforward application of the KKT conditions. This lemma says that maximizing the modified reward can be equivalently viewed as restricting the set of policies to those that do not exploit the dynamics. Next, we will show that policies that do not exploit the dynamics have an expected (entropy-regularized) reward that is similar in the source and target domains:
\begin{lemma} \label{lemma:dynamics-rewards}
Let policy $\pi \in \Pi_\text{no exploit}$ be given, and let $R_\text{max}$ be the maximum (entropy-regularized) return of any trajectory. Then the following inequality holds:
\begin{equation*}
 \left| \E_{p_\text{source}} \left[ \sum r(s_t, a_t) + \gH_\pi[a_t \mid s_t] \right] - \E_{p_\text{target}} \left[ \sum r(s_t, a_t) + \gH_\pi[a_t \mid s_t] \right] \right| \le 2 R_\text{max} \sqrt{\epsilon / 2}.
\end{equation*}
\end{lemma}
This Lemma proves that all policies in $\Pi_\text{no exploit}$ satisfy the same condition as the optimal policy (Assumption~\ref{assumption:opt}).
\begin{proof}
To simplify notation, define $\tilde{r}(s, a) = r(s, a) - \log \pi(a \mid s)$. We then apply Holder's inequality and Pinsker's inequality to obtain the desired result:
\begin{align*}
    \E_{p_\text{source}} \left[\sum \tilde{r}(s_t, a_t) \right] - \E_{p_\text{target}} \left[\sum \tilde{r}(s_t, a_t) \right] &= \sum_\tau (p_\text{source}(\tau) - p_\text{target}(\tau)) \left(\sum \tilde{r}(s_t, a_t) \right) \\
    & \le \|\sum \tilde{r}(s_t, a_t) \|_\infty \cdot \| p_\text{source}(\tau) - p_\text{target}(\tau) \|_1 \\
    & \le \left( \max_\tau \sum r(s_t, a_t) \right) \cdot 2 \sqrt{\frac{1}{2} \kl{p_\text{source}(\tau)}{p_\text{target}(\tau)}} \\
    & \le 2 R_\text{max} \sqrt{\epsilon / 2}.
\end{align*}
\end{proof}

We restate our main result:
\renewcommand\thetheorem{4.1}
\begin{theorem}[Repeated from main text]
Let $\pi_\text{DARC}^*$ be the policy that maximizes the modified (entropy-regularized) reward in the source domain, let $\pi^*$ be the policy that maximizes the (unmodified, entropy-regularized) reward in the target domain, and assume that $\pi^*$ satisfies Assumption~\ref{assumption:opt}. Then $\pi_\text{DARC}^*$ receives near-optimal (entropy-regularized) reward on the target domain:
\begin{equation*}
    \E_{p_\text{target}, \pi_\text{DARC}^*} \left[ \sum r(s_t, a_t) + \gH[a_t \mid s_t] \right] \ge \E_{p_\text{target}, \pi^*} \left[ \sum r(s_t, a_t) + \gH[a_t \mid s_t] \right] - 4 R_\text{max}\sqrt{\epsilon / 2}.
\end{equation*}
\end{theorem}

\begin{proof}
Assumption~\ref{assumption:opt} guarantees that the optimal policy for the target domain, $\pi^*$, lies within $\Pi_\text{no exploit}$. Among all policies in $\Pi_\text{no exploit}$, $\pi_\text{DARC}^*$ is (by definition) the one that receives highest reward on the source dynamics, so
\begin{equation*}
    \E_{p_\text{source}, \pi_\text{DARC}^*} \left[ \sum r(s_t, a_t) + \gH[a_t \mid s_t] \right] \ge \E_{p_\text{source}, \pi^*} \left[ \sum r(s_t, a_t) + \gH[a_t \mid s_t] \right]. \label{eq:proof-1}
\end{equation*}
Since the both $\pi_\text{DARC}^*$ and $\pi^*$ lie inside the constraint set, Lemma~\ref{lemma:dynamics-rewards} dictates that their rewards on the target domain differ by at most $2 R_\text{max}\sqrt{\epsilon / 2}$ from their source domain rewards. In the worst case, the reward for $\pi_\text{DARC}^*$ decreases by this amount and the reward for $\pi^*$ increases by this amount:
\begin{align*}
 \E_{p_\text{source}, \pi_\text{DARC}^*} \left[ \sum r(s_t, a_t) + \gH[a_t \mid s_t] \right] & \le \E_{p_\text{target}, \pi_\text{DARC}^*} \left[ \sum r(s_t, a_t) + \gH[a_t \mid s_t] \right] + 2 R_\text{max} \sqrt{\epsilon / 2} \\
 \E_{p_\text{source}, \pi^*} \left[ \sum r(s_t, a_t) + \gH[a_t \mid s_t] \right] & \ge \E_{p_\text{target}, \pi^*} \left[ \sum r(s_t, a_t) + \gH[a_t \mid s_t] \right] - 2 R_\text{max} \sqrt{\epsilon / 2} \\
\end{align*}
Substituting these inequalities on the LHS and RHS of Eq.~\ref{eq:proof-1} and rearranging terms, we obtain the desired result.
\end{proof}

\section{The Special Case of an Observation Model}
\label{appendix:observation-model}

To highlight the relationship between domain adaptation of dynamics versus observations, we now consider a special case. In this subsection, we will assume that the state $s_t \triangleq (z_t, o_t)$ is a combination of the system latent state $z_t$ (e.g., the poses of all objects in a scene) and an observation $o_t$ (e.g., a camera observation). We will define $q(o_t \mid z_t)$ and $p(o_t \mid z_t)$ as the \emph{observation models} for the source and target domains. In this special case, we can decompose the KL objective (Eq.~\ref{eq:kl}) into three terms:
\begin{align*}
    \kl{q}{p} = -\E_q \bigg[\sum_t &\underbrace{r(s_t, a_t) + \gH_\pi[a_t \mid s_t]}_{\text{MaxEnt RL objective}} + \underbrace{\log p_{\text{target}}(o_t \mid z_t) - \log p_{\text{source}}(o_t \mid z_t)}_{\text{Observation Adaptation}} \\
    & +\underbrace{\log p_{\text{target}}(z_{t+1} \mid z_t, a_t) - \log p_{\text{source}}(z_{t+1} \mid z_t, a_t)}_{\text{Dynamics Adaptation}}  \bigg].
\end{align*}
Prior methods that perform observation adaptation~\citep{bousmalis2018using, gamrian2018transfer} effectively minimize the observation adaptation term,\footnote{\citet{tiao2018cycle} show that observation adaptation using CycleGan~\citep{zhu2017unpaired} minimizes a Jensen-Shannon divergence. Assuming sufficiently expressive models, the Jensen-Shannon divergence and the reverse KL divergence above have the same optimum.}
but ignore the effect of dynamics. In contrast, the $\Delta r$ reward correction in our method provides one method to address both dynamics and observations. These approaches could be combined; we leave this as future work.

\clearpage
\section{Additional Experiments}
\label{appendix:ablation}

\begin{figure}[t]
    \centering
    \includegraphics[width=\linewidth]{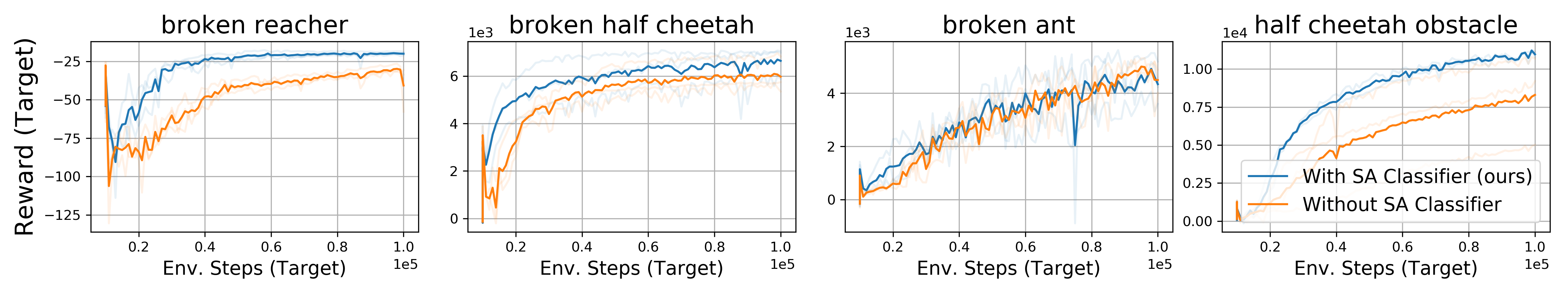}
    \caption{\textbf{Importance of using two classifiers}: Results of the ablation experiment from Fig.~\ref{fig:ablation} (left) on all environments.}
    \label{fig:more-classifier}
\end{figure}

\begin{figure}[t]
    \centering
    \includegraphics[width=\linewidth]{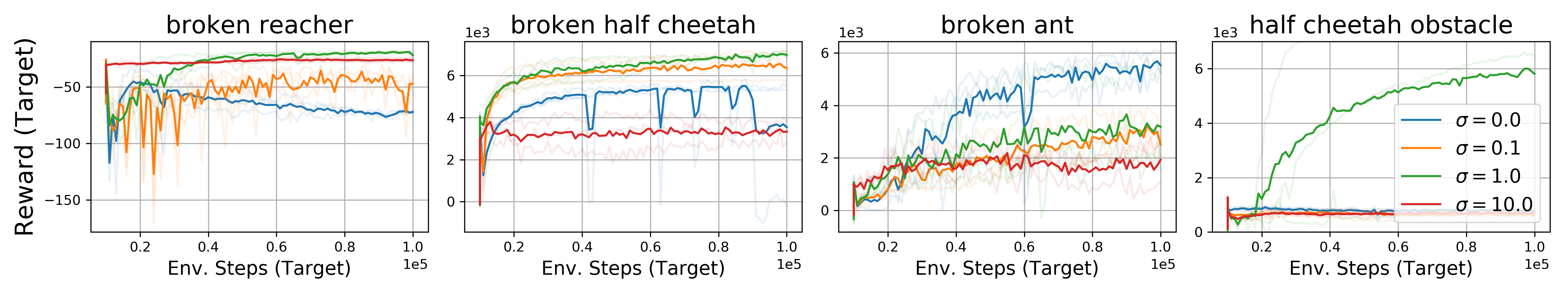}
    \caption{\textbf{Importance of regularizing the classifiers}: Results of the ablation experiment from Fig.~\ref{fig:ablation} (right) on all environments.}
    \label{fig:more-noise}
\end{figure}

\begin{wrapfigure}[10]{R}{0.5\textwidth}
    \centering
    \vspace{-2em}
    \includegraphics[width=\linewidth]{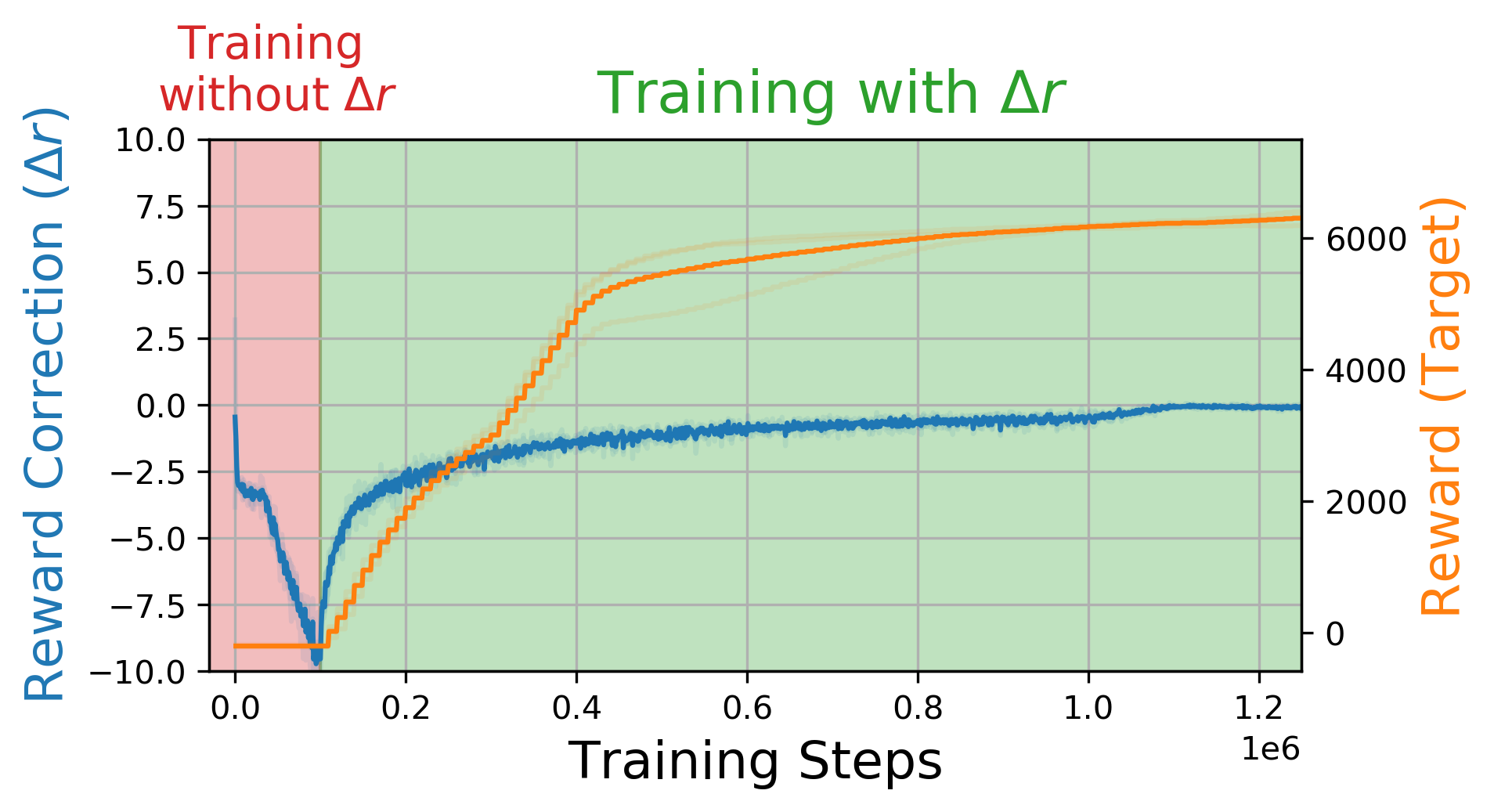}
    \vspace{-1em}
    \caption{Copy of Fig.~\ref{fig:delta-r} overlaid with the target domain reward.}
    \label{fig:delta-r-reward}
\end{wrapfigure}

Figures~\ref{fig:more-classifier} and~\ref{fig:more-noise} show the results of the ablation experiment from Fig.~\ref{fig:ablation} run on all environments. The results support our conclusion in the main text regarding the importance of using two classifiers and using input noise.
Figure~\ref{fig:delta-r-reward} is a copy of Fig.~\ref{fig:delta-r} from the main text, modified to also show the agent's reward on the target domain. We observe that the reward does not start increasing until we start using DARC.

\section{Experiment Details and Hyperparameters}

Our implementation of DARC is built on top of the implementation of SAC from~\citet{TFAgents}. Unless otherwise specified, all hyperparameters are taken from~\citet{TFAgents}. All neural networks (actor, critics, and classifiers) have two hidden layers with 256-units each and ReLU activations. Since we ultimately will use the \emph{difference} in the predictions of the two classifiers, we use a residual parametrization for the SAS classifier $q(\text{target} \mid s_t, a_t, s_{t+1})$. Using $f_\text{SAS}(s_t, a_t, s_{t+1}), f_\text{SA}(s_t, a_t) \in \mathbbm{R}^2$ to denote the outputs of the two classifier networks, we compute the classifier predictions as follows:
\begin{align*}
    q_{\theta_{\text{SA}}}(\cdot \mid s_t, a_t) &= \textsc{SoftMax}(f_\text{SA}(s_t, a_t)) \\
    q_{\theta_{\text{SAS}}}(\cdot \mid s_t, a_t, s_{t+1}) &= \textsc{SoftMax}(f_\text{SAS}(s_t, a_t, s_{t+1}) + f_\text{SA}(s_t, a_t))
\end{align*}
For the SAS classifier we propagate gradients back through both networks parameters, $\theta_\text{SAS}$ and $\theta_\text{SA}$. Both classifiers use Gaussian input noise with $\sigma = 1$. Optimization of all networks is done with Adam~\citep{kingma2014adam} with a learning rate of 3e-4 and batch size of 128. Most experiments with DARC collected 1 step in the target domain every 10 steps in the source domain (i.e., $r = 10$). The one exception is the half cheetah obstacle domain, where we tried increasing $r$ beyond 10 to 30, 100, 300, and 1000. We found a large benefit from increasing $r$ to 30 and 100, but did not run the other experiments long enough to draw any conclusions. Fig.~\ref{fig:broken-robots} uses $r = 30$ for half cheetah obstacle. We did not tune this parameter, and expect that tuning it would result in significant improvements in sample efficiency.

We found that DARC was slightly more stable if we warm-started the method by applying RL on the source task \emph{without} $\Delta r$ for the first $t_\text{warmup}$ iterations. We used $t_\text{warmup} = 1e5$ for all tasks except the broken reacher, where we used $t_\text{warmup} = 2e5$. This discrepancy was caused by a typo in an experiment, and subsequent experiments found that DARC is relatively robust to different values of~$t_\text{warmup}$; we did not tune this parameter.

\subsection{Baselines}

The \textbf{RL on Source} and \textbf{RL on Target} baselines are implemented identically to our method, with the exception that $\Delta r$ is not added to the reward function. The \textbf{RL on Target (10x)} is identical to RL on Target, with the exception that we take 10 gradient steps per environment interaction (instead of 1). The \textbf{Importance Weighting} baseline estimates the importance weights as $p_\text{target}(s_{t+1} \mid s_t, a_t) / p_\text{source}(s_{t+1} \mid s_t, a_t) \approx \exp(\Delta r)$. The importance weight is used to weight transitions in the SAC actor and critic losses.

\paragraph{PETS~\citep{chua2018deep}} The \textbf{PETS} baseline is implemented using the default configurations used by \citep{chua2018deep} for the environments evaluated. The \texttt{broken-half-cheetah} environment uses the hyperparameters as used by the \texttt{half-cheetah} environment in \citep{chua2018deep}. The \texttt{broken-ant} environment uses the same set of hyperparameters, namely: task horizon = 1000, number of training iterations = 300, number of planning (real) steps per iteration = 30, number of particles to be used in particle propagation methods = 20. The PETS codebase can be found at \url{https://github.com/kchua/handful-of-trials}.

\paragraph{MBPO~\citep{janner2019trust}} We used the authors implementation with the default hyperparameters: \url{https://github.com/JannerM/mbpo}.
We kept the environment configurations the same as their default unmodified MuJoCo environments, except for the domain and task name. We added our custom environment xmls in \texttt{mbpo/env/assets/} folder, and their corresponding environment python files in the \texttt{mbpo/env/ folder}. Their static files were added under \texttt{mbpo/static/}. These environments can be registered as gym environments in the init file under \texttt{mbpo\_odrl/mbpo/env/} or can be initialized directly in \texttt{softlearning/environments/adapters/gym adapter.py}. We set the time limit to \texttt{max\_episode\_steps=1000} for the Broken Half Cheetah, Broken Ant and Half Cheetah Obstacle environments and to 100 for the Broken Reacher environment.

\subsection{Environments}

\paragraph{Broken Reacher} This environment uses the 7DOF robot arm from the Pusher environment in OpenAI Gym. The observation space is the position and velocities of all joints and the goal. The reward function is 
\begin{equation*}
    r(s, a) = -\frac{1}{2}\|s_\text{end effector} - s_\text{goal}\|_2 - \frac{1}{10}\|a\|_2^2,
\end{equation*}
and episodes are 100 steps long. In the target domain the 2nd joint (0-indexed) is broken: zero torque is applied to this joint, regardless of the commanded torque. 

\paragraph{Broken Half Cheetah}
This environment is based on the HalfCheetah environment in OpenAI Gym. We modified the observation to include the agent's global X coordinate so the agent can infer its relative position to the obstacle. Episodes are 1000 steps long. In the target domain the 0th joint (0-indexed) is broken: zero torque is applied to this joint, regardless of the commanded torque. 

\paragraph{Broken Ant}
This environment is based on the Ant environment in OpenAI Gym. We use the standard termination condition and cap the maximum episode length at 1000 steps. In the target domain the 3rd joint (0-indexed) is broken: zero torque is applied to this joint, regardless of the commanded torque.

In all the broken joint environments, we choose which joint to break to computing which joint caused the ``RL on Source'' baseline to perform worst on the target domain, as compared with the ``RL on Target'' baseline.

\paragraph{Half Cheetah Obstacle}
This environment is based on the HalfCheetah environment in OpenAI Gym. Episodes are 1000 steps long. We modified the standard reward function to use the absolute value in place of the velocity, resulting the following reward function:
\begin{equation*}
    r(s, a) = s_\text{x vel} \cdot \Delta t - \|a\|_2^2,
\end{equation*}
where $s_\text{x vel}$ is the velocity of the agent along the forward-aft axis and $\Delta t = 0.01$ is the time step of the simulator. In the target domain, we added a wall at $x = -3m$, roughly 3 meters behind the agent.

\paragraph{Humanoid}
Used for the experiment in Fig.~\ref{fig:falling}, we used a modified version of Humanoid from OpenAI Gym. The source domain modified this environment to ignore the default termination condition and instead terminate after exactly 300 time steps. The target domain uses the unmodified environment, which terminates when the agent falls.

\subsection{Figures}
Unless otherwise noted, all experiments were run with three random seeds. Figures showing learning curves (Figures~\ref{fig:broken-robots},~\ref{fig:ablation},~\ref{fig:delta-r},~\ref{fig:more-classifier}, and~\ref{fig:more-noise}) plot the \emph{mean} over the three random seeds, and also plot the results for each individual random seed with semi-transparent lines.

\subsection{Archery Experiment}
\label{appendix:archery}
We used a simple physics model for the archery experiment. The target was located 70m North of the agent, and wind was applied along the East-West axis. The system dynamics:
\begin{equation*}
    s_{t+1} = 70 \sin(\theta) + f / \cos(\theta)^2 \qquad \begin{cases} f \sim \gN(\mu=1, \sigma=1) & \text{in the source domain} \\ f \sim \gN(\mu=0, \sigma=0.3) & \text{in the target domain} \end{cases}
\end{equation*}

We trained the classifier by sampling $\theta \sim \gU[-2, 2]$ (measured in degrees) for 10k episodes in the source domain and 10k episodes in the target domain. The classifier was a neural network with 1 hidden layer with 32 hidden units and ReLU activation. We optimized the classifier using the Adam optimizer with a learning rate of 3e-3 and a batch size of 1024. We trained until the validation loss increased for 3 consecutive epochs, which took 16 epochs in our experiment. We generated Fig.~\ref{fig:arrow} by sampling 10k episodes for each value of $\theta$ and aggregating the rewards using \mbox{$J(\theta) = \log \E_{p(s' \mid \theta)}[\exp(r(s'))]$}. We found that aggregating rewards by taking the mean did not yield meaningful results, perhaps because the mean corresponds to a (possibly loose) lower bound on $J$ (see Appendix~\ref{appendix:risk-sensitive}).


{\footnotesize
\begin{thebibliography}{85}
\providecommand{\natexlab}[1]{#1}
\providecommand{\url}[1]{\texttt{#1}}
\expandafter\ifx\csname urlstyle\endcsname\relax
  \providecommand{\doi}[1]{doi: #1}\else
  \providecommand{\doi}{doi: \begingroup \urlstyle{rm}\Url}\fi

\bibitem[Abdolmaleki et~al.(2018)Abdolmaleki, Springenberg, Tassa, Munos,
  Heess, and Riedmiller]{abdolmaleki2018maximum}
Abbas Abdolmaleki, Jost~Tobias Springenberg, Yuval Tassa, Remi Munos, Nicolas
  Heess, and Martin Riedmiller.
\newblock Maximum a posteriori policy optimisation.
\newblock \emph{arXiv preprint arXiv:1806.06920}, 2018.

\bibitem[Achiam et~al.(2017)Achiam, Held, Tamar, and
  Abbeel]{achiam2017constrained}
Joshua Achiam, David Held, Aviv Tamar, and Pieter Abbeel.
\newblock Constrained policy optimization.
\newblock In \emph{Proceedings of the 34th International Conference on Machine
  Learning-Volume 70}, pp.\  22--31. JMLR. org, 2017.

\bibitem[Berkenkamp et~al.(2017)Berkenkamp, Turchetta, Schoellig, and
  Krause]{berkenkamp2017safe}
Felix Berkenkamp, Matteo Turchetta, Angela Schoellig, and Andreas Krause.
\newblock Safe model-based reinforcement learning with stability guarantees.
\newblock In \emph{Advances in neural information processing systems}, pp.\
  908--918, 2017.

\bibitem[Bickel et~al.(2007)Bickel, Br{\"u}ckner, and
  Scheffer]{bickel2007discriminative}
Steffen Bickel, Michael Br{\"u}ckner, and Tobias Scheffer.
\newblock Discriminative learning for differing training and test
  distributions.
\newblock In \emph{Proceedings of the 24th international conference on Machine
  learning}, pp.\  81--88, 2007.

\bibitem[Bousmalis et~al.(2018)Bousmalis, Irpan, Wohlhart, Bai, Kelcey,
  Kalakrishnan, Downs, Ibarz, Pastor, Konolige, et~al.]{bousmalis2018using}
Konstantinos Bousmalis, Alex Irpan, Paul Wohlhart, Yunfei Bai, Matthew Kelcey,
  Mrinal Kalakrishnan, Laura Downs, Julian Ibarz, Peter Pastor, Kurt Konolige,
  et~al.
\newblock Using simulation and domain adaptation to improve efficiency of deep
  robotic grasping.
\newblock In \emph{2018 IEEE International Conference on Robotics and
  Automation (ICRA)}, pp.\  4243--4250. IEEE, 2018.

\bibitem[Brockman et~al.(2016)Brockman, Cheung, Pettersson, Schneider,
  Schulman, Tang, and Zaremba]{brockman2016openai}
Greg Brockman, Vicki Cheung, Ludwig Pettersson, Jonas Schneider, John Schulman,
  Jie Tang, and Wojciech Zaremba.
\newblock Openai gym.
\newblock \emph{arXiv preprint arXiv:1606.01540}, 2016.

\bibitem[Chebotar et~al.(2019)Chebotar, Handa, Makoviychuk, Macklin, Issac,
  Ratliff, and Fox]{chebotar2019closing}
Yevgen Chebotar, Ankur Handa, Viktor Makoviychuk, Miles Macklin, Jan Issac,
  Nathan Ratliff, and Dieter Fox.
\newblock Closing the sim-to-real loop: Adapting simulation randomization with
  real world experience.
\newblock In \emph{2019 International Conference on Robotics and Automation
  (ICRA)}, pp.\  8973--8979. IEEE, 2019.

\bibitem[Chua et~al.(2018)Chua, Calandra, McAllister, and Levine]{chua2018deep}
Kurtland Chua, Roberto Calandra, Rowan McAllister, and Sergey Levine.
\newblock Deep reinforcement learning in a handful of trials using
  probabilistic dynamics models.
\newblock In \emph{Advances in Neural Information Processing Systems}, pp.\
  4754--4765, 2018.

\bibitem[Clavera et~al.(2018)Clavera, Nagabandi, Fearing, Abbeel, Levine, and
  Finn]{clavera2018learning}
Ignasi Clavera, Anusha Nagabandi, Ronald~S Fearing, Pieter Abbeel, Sergey
  Levine, and Chelsea Finn.
\newblock Learning to adapt: Meta-learning for model-based control.
\newblock \emph{arXiv preprint arXiv:1803.11347}, 3, 2018.

\bibitem[Cortes \& Mohri(2014)Cortes and Mohri]{cortes2014domain}
Corinna Cortes and Mehryar Mohri.
\newblock Domain adaptation and sample bias correction theory and algorithm for
  regression.
\newblock \emph{Theoretical Computer Science}, 519:\penalty0 103--126, 2014.

\bibitem[Csurka(2017)]{csurka2017domain}
Gabriela Csurka.
\newblock Domain adaptation for visual applications: A comprehensive survey.
\newblock \emph{arXiv preprint arXiv:1702.05374}, 2017.

\bibitem[Cutler et~al.(2014)Cutler, Walsh, and How]{cutler2014reinforcement}
Mark Cutler, Thomas~J Walsh, and Jonathan~P How.
\newblock Reinforcement learning with multi-fidelity simulators.
\newblock In \emph{2014 IEEE International Conference on Robotics and
  Automation (ICRA)}, pp.\  3888--3895. IEEE, 2014.

\bibitem[Dann et~al.(2014)Dann, Neumann, Peters, et~al.]{dann2014policy}
Christoph Dann, Gerhard Neumann, Jan Peters, et~al.
\newblock Policy evaluation with temporal differences: A survey and comparison.
\newblock \emph{Journal of Machine Learning Research}, 15:\penalty0 809--883,
  2014.

\bibitem[Deisenroth \& Rasmussen(2011)Deisenroth and
  Rasmussen]{deisenroth2011pilco}
Marc Deisenroth and Carl~E Rasmussen.
\newblock Pilco: A model-based and data-efficient approach to policy search.
\newblock In \emph{Proceedings of the 28th International Conference on machine
  learning (ICML-11)}, pp.\  465--472, 2011.

\bibitem[Duan et~al.(2016)Duan, Schulman, Chen, Bartlett, Sutskever, and
  Abbeel]{duan2016rl}
Yan Duan, John Schulman, Xi~Chen, Peter~L Bartlett, Ilya Sutskever, and Pieter
  Abbeel.
\newblock Rl\^{}2: Fast reinforcement learning via slow reinforcement learning.
\newblock \emph{arXiv preprint arXiv:1611.02779}, 2016.

\bibitem[Dud{\'\i}k et~al.(2011)Dud{\'\i}k, Langford, and Li]{dudik2011doubly}
Miroslav Dud{\'\i}k, John Langford, and Lihong Li.
\newblock Doubly robust policy evaluation and learning.
\newblock \emph{arXiv preprint arXiv:1103.4601}, 2011.

\bibitem[Eysenbach et~al.(2017)Eysenbach, Gu, Ibarz, and
  Levine]{eysenbach2017leave}
Benjamin Eysenbach, Shixiang Gu, Julian Ibarz, and Sergey Levine.
\newblock Leave no trace: Learning to reset for safe and autonomous
  reinforcement learning.
\newblock \emph{arXiv preprint arXiv:1711.06782}, 2017.

\bibitem[Farchy et~al.(2013)Farchy, Barrett, MacAlpine, and
  Stone]{farchy2013humanoid}
Alon Farchy, Samuel Barrett, Patrick MacAlpine, and Peter Stone.
\newblock Humanoid robots learning to walk faster: From the real world to
  simulation and back.
\newblock In \emph{Proceedings of the 2013 international conference on
  Autonomous agents and multi-agent systems}, pp.\  39--46, 2013.

\bibitem[Feldbaum(1960)]{feldbaum1960dual}
AA~Feldbaum.
\newblock Dual control theory. i.
\newblock \emph{Avtomatika i Telemekhanika}, 21\penalty0 (9):\penalty0
  1240--1249, 1960.

\bibitem[Fernando et~al.(2013)Fernando, Habrard, Sebban, and
  Tuytelaars]{fernando2013unsupervised}
Basura Fernando, Amaury Habrard, Marc Sebban, and Tinne Tuytelaars.
\newblock Unsupervised visual domain adaptation using subspace alignment.
\newblock In \emph{Proceedings of the IEEE international conference on computer
  vision}, pp.\  2960--2967, 2013.

\bibitem[Fu et~al.(2017)Fu, Luo, and Levine]{fu2017learning}
Justin Fu, Katie Luo, and Sergey Levine.
\newblock Learning robust rewards with adversarial inverse reinforcement
  learning.
\newblock \emph{arXiv preprint arXiv:1710.11248}, 2017.

\bibitem[Fujimoto et~al.(2018)Fujimoto, Meger, and Precup]{fujimoto2018off}
Scott Fujimoto, David Meger, and Doina Precup.
\newblock Off-policy deep reinforcement learning without exploration.
\newblock \emph{arXiv preprint arXiv:1812.02900}, 2018.

\bibitem[Gamrian \& Goldberg(2018)Gamrian and Goldberg]{gamrian2018transfer}
Shani Gamrian and Yoav Goldberg.
\newblock Transfer learning for related reinforcement learning tasks via
  image-to-image translation.
\newblock \emph{arXiv preprint arXiv:1806.07377}, 2018.

\bibitem[Ganin et~al.(2016)Ganin, Ustinova, Ajakan, Germain, Larochelle,
  Laviolette, Marchand, and Lempitsky]{ganin2016domain}
Yaroslav Ganin, Evgeniya Ustinova, Hana Ajakan, Pascal Germain, Hugo
  Larochelle, Fran{\c{c}}ois Laviolette, Mario Marchand, and Victor Lempitsky.
\newblock Domain-adversarial training of neural networks.
\newblock \emph{The Journal of Machine Learning Research}, 17\penalty0
  (1):\penalty0 2096--2030, 2016.

\bibitem[Guadarrama et~al.(2018)Guadarrama, Korattikara, Ramirez, Castro,
  Holly, Fishman, Wang, Gonina, Harris, Vanhoucke, et~al.]{TFAgents}
Sergio Guadarrama, Anoop Korattikara, Oscar Ramirez, Pablo Castro, Ethan Holly,
  Sam Fishman, Ke~Wang, Ekaterina Gonina, Chris Harris, Vincent Vanhoucke,
  et~al.
\newblock Tf-agents: A library for reinforcement learning in tensorflow, 2018.

\bibitem[Haarnoja et~al.(2018)Haarnoja, Zhou, Abbeel, and
  Levine]{haarnoja2018soft}
Tuomas Haarnoja, Aurick Zhou, Pieter Abbeel, and Sergey Levine.
\newblock Soft actor-critic: Off-policy maximum entropy deep reinforcement
  learning with a stochastic actor.
\newblock \emph{arXiv preprint arXiv:1801.01290}, 2018.

\bibitem[Hafner et~al.(2018)Hafner, Lillicrap, Fischer, Villegas, Ha, Lee, and
  Davidson]{hafner2018learning}
Danijar Hafner, Timothy Lillicrap, Ian Fischer, Ruben Villegas, David Ha,
  Honglak Lee, and James Davidson.
\newblock Learning latent dynamics for planning from pixels.
\newblock \emph{arXiv preprint arXiv:1811.04551}, 2018.

\bibitem[Higgins et~al.(2017)Higgins, Pal, Rusu, Matthey, Burgess, Pritzel,
  Botvinick, Blundell, and Lerchner]{higgins2017darla}
Irina Higgins, Arka Pal, Andrei Rusu, Loic Matthey, Christopher Burgess,
  Alexander Pritzel, Matthew Botvinick, Charles Blundell, and Alexander
  Lerchner.
\newblock Darla: Improving zero-shot transfer in reinforcement learning.
\newblock In \emph{Proceedings of the 34th International Conference on Machine
  Learning-Volume 70}, pp.\  1480--1490. JMLR. org, 2017.

\bibitem[Ho \& Ermon(2016)Ho and Ermon]{ho2016generative}
Jonathan Ho and Stefano Ermon.
\newblock Generative adversarial imitation learning.
\newblock In \emph{Advances in neural information processing systems}, pp.\
  4565--4573, 2016.

\bibitem[Hoffman et~al.(2016)Hoffman, Wang, Yu, and Darrell]{hoffman2016fcns}
Judy Hoffman, Dequan Wang, Fisher Yu, and Trevor Darrell.
\newblock Fcns in the wild: Pixel-level adversarial and constraint-based
  adaptation.
\newblock \emph{arXiv preprint arXiv:1612.02649}, 2016.

\bibitem[Janner et~al.(2019)Janner, Fu, Zhang, and Levine]{janner2019trust}
Michael Janner, Justin Fu, Marvin Zhang, and Sergey Levine.
\newblock When to trust your model: Model-based policy optimization.
\newblock In \emph{Advances in Neural Information Processing Systems}, pp.\
  12498--12509, 2019.

\bibitem[Jaques et~al.(2017)Jaques, Gu, Bahdanau, Hern{\'a}ndez-Lobato, Turner,
  and Eck]{jaques2017sequence}
Natasha Jaques, Shixiang Gu, Dzmitry Bahdanau, Jos{\'e}~Miguel
  Hern{\'a}ndez-Lobato, Richard~E Turner, and Douglas Eck.
\newblock Sequence tutor: Conservative fine-tuning of sequence generation
  models with kl-control.
\newblock In \emph{Proceedings of the 34th International Conference on Machine
  Learning-Volume 70}, pp.\  1645--1654. JMLR. org, 2017.

\bibitem[Kappen(2005)]{kappen2005path}
Hilbert~J Kappen.
\newblock Path integrals and symmetry breaking for optimal control theory.
\newblock \emph{Journal of statistical mechanics: theory and experiment},
  2005\penalty0 (11):\penalty0 P11011, 2005.

\bibitem[Killian et~al.(2017)Killian, Daulton, Konidaris, and
  Doshi-Velez]{killian2017robust}
Taylor~W Killian, Samuel Daulton, George Konidaris, and Finale Doshi-Velez.
\newblock Robust and efficient transfer learning with hidden parameter markov
  decision processes.
\newblock In \emph{Advances in neural information processing systems}, pp.\
  6250--6261, 2017.

\bibitem[Kingma \& Ba(2014)Kingma and Ba]{kingma2014adam}
Diederik~P Kingma and Jimmy Ba.
\newblock Adam: A method for stochastic optimization.
\newblock \emph{arXiv preprint arXiv:1412.6980}, 2014.

\bibitem[Koller \& Friedman(2009)Koller and Friedman]{koller2009probabilistic}
Daphne Koller and Nir Friedman.
\newblock \emph{Probabilistic graphical models: principles and techniques}.
\newblock MIT press, 2009.

\bibitem[Koos et~al.(2012)Koos, Mouret, and Doncieux]{koos2012transferability}
Sylvain Koos, Jean-Baptiste Mouret, and St{\'e}phane Doncieux.
\newblock The transferability approach: Crossing the reality gap in
  evolutionary robotics.
\newblock \emph{IEEE Transactions on Evolutionary Computation}, 17\penalty0
  (1):\penalty0 122--145, 2012.

\bibitem[Kouw \& Loog(2019)Kouw and Loog]{kouw2019review}
Wouter~Marco Kouw and Marco Loog.
\newblock A review of domain adaptation without target labels.
\newblock \emph{IEEE transactions on pattern analysis and machine
  intelligence}, 2019.

\bibitem[Lazaric(2008)]{lazaric2008knowledge}
Alessandro Lazaric.
\newblock \emph{Knowledge transfer in reinforcement learning}.
\newblock PhD thesis, PhD thesis, Politecnico di Milano, 2008.

\bibitem[Levine(2018)]{levine2018reinforcement}
Sergey Levine.
\newblock Reinforcement learning and control as probabilistic inference:
  Tutorial and review.
\newblock \emph{arXiv preprint arXiv:1805.00909}, 2018.

\bibitem[Levine et~al.(2018)Levine, Pastor, Krizhevsky, Ibarz, and
  Quillen]{levine2018learning}
Sergey Levine, Peter Pastor, Alex Krizhevsky, Julian Ibarz, and Deirdre
  Quillen.
\newblock Learning hand-eye coordination for robotic grasping with deep
  learning and large-scale data collection.
\newblock \emph{The International Journal of Robotics Research}, 37\penalty0
  (4-5):\penalty0 421--436, 2018.

\bibitem[Lipton et~al.(2018)Lipton, Wang, and Smola]{lipton2018detecting}
Zachary~C Lipton, Yu-Xiang Wang, and Alex Smola.
\newblock Detecting and correcting for label shift with black box predictors.
\newblock \emph{arXiv preprint arXiv:1802.03916}, 2018.

\bibitem[Ljung(1999)]{ljung1999system}
Lennart Ljung.
\newblock System identification.
\newblock \emph{Wiley encyclopedia of electrical and electronics engineering},
  pp.\  1--19, 1999.

\bibitem[Madden \& Howley(2004)Madden and Howley]{madden2004transfer}
Michael~G Madden and Tom Howley.
\newblock Transfer of experience between reinforcement learning environments
  with progressive difficulty.
\newblock \emph{Artificial Intelligence Review}, 21\penalty0 (3-4):\penalty0
  375--398, 2004.

\bibitem[Mihatsch \& Neuneier(2002)Mihatsch and Neuneier]{mihatsch2002risk}
Oliver Mihatsch and Ralph Neuneier.
\newblock Risk-sensitive reinforcement learning.
\newblock \emph{Machine learning}, 49\penalty0 (2-3):\penalty0 267--290, 2002.

\bibitem[Mishra et~al.(2017)Mishra, Rohaninejad, Chen, and
  Abbeel]{mishra2017simple}
Nikhil Mishra, Mostafa Rohaninejad, Xi~Chen, and Pieter Abbeel.
\newblock A simple neural attentive meta-learner.
\newblock \emph{arXiv preprint arXiv:1707.03141}, 2017.

\bibitem[Mohamed \& Lakshminarayanan(2016)Mohamed and
  Lakshminarayanan]{mohamed2016learning}
Shakir Mohamed and Balaji Lakshminarayanan.
\newblock Learning in implicit generative models.
\newblock \emph{arXiv preprint arXiv:1610.03483}, 2016.

\bibitem[Munos et~al.(2016)Munos, Stepleton, Harutyunyan, and
  Bellemare]{munos2016safe}
R{\'e}mi Munos, Tom Stepleton, Anna Harutyunyan, and Marc Bellemare.
\newblock Safe and efficient off-policy reinforcement learning.
\newblock In \emph{Advances in Neural Information Processing Systems}, pp.\
  1054--1062, 2016.

\bibitem[Peng et~al.(2018)Peng, Andrychowicz, Zaremba, and Abbeel]{peng2018sim}
Xue~Bin Peng, Marcin Andrychowicz, Wojciech Zaremba, and Pieter Abbeel.
\newblock Sim-to-real transfer of robotic control with dynamics randomization.
\newblock In \emph{2018 IEEE international conference on robotics and
  automation (ICRA)}, pp.\  1--8. IEEE, 2018.

\bibitem[Perkins et~al.(1999)Perkins, Precup, et~al.]{perkins1999using}
Theodore~J Perkins, Doina Precup, et~al.
\newblock Using options for knowledge transfer in reinforcement learning.
\newblock \emph{University of Massachusetts, Amherst, MA, USA, Tech. Rep},
  1999.

\bibitem[Rajeswaran et~al.(2016)Rajeswaran, Ghotra, Ravindran, and
  Levine]{rajeswaran2016epopt}
Aravind Rajeswaran, Sarvjeet Ghotra, Balaraman Ravindran, and Sergey Levine.
\newblock Epopt: Learning robust neural network policies using model ensembles.
\newblock \emph{arXiv preprint arXiv:1610.01283}, 2016.

\bibitem[Rakelly et~al.(2019)Rakelly, Zhou, Finn, Levine, and
  Quillen]{rakelly2019efficient}
Kate Rakelly, Aurick Zhou, Chelsea Finn, Sergey Levine, and Deirdre Quillen.
\newblock Efficient off-policy meta-reinforcement learning via probabilistic
  context variables.
\newblock In \emph{International conference on machine learning}, pp.\
  5331--5340, 2019.

\bibitem[Ravindran \& Barto(2004)Ravindran and Barto]{ravindran2004algebraic}
Balaraman Ravindran and Andrew~G Barto.
\newblock \emph{An algebraic approach to abstraction in reinforcement
  learning}.
\newblock PhD thesis, University of Massachusetts at Amherst, 2004.

\bibitem[Rawlik et~al.(2013)Rawlik, Toussaint, and
  Vijayakumar]{rawlik2013stochastic}
Konrad Rawlik, Marc Toussaint, and Sethu Vijayakumar.
\newblock On stochastic optimal control and reinforcement learning by
  approximate inference.
\newblock In \emph{Twenty-Third International Joint Conference on Artificial
  Intelligence}, 2013.

\bibitem[Ross \& Bagnell(2012)Ross and Bagnell]{ross2012agnostic}
Stephane Ross and J~Andrew Bagnell.
\newblock Agnostic system identification for model-based reinforcement
  learning.
\newblock \emph{arXiv preprint arXiv:1203.1007}, 2012.

\bibitem[Sadeghi \& Levine(2016)Sadeghi and Levine]{sadeghi2016cad2rl}
Fereshteh Sadeghi and Sergey Levine.
\newblock Cad2rl: Real single-image flight without a single real image.
\newblock \emph{arXiv preprint arXiv:1611.04201}, 2016.

\bibitem[Sastry \& Isidori(1989)Sastry and Isidori]{sastry1989adaptive}
Sosale~Shankara Sastry and Alberto Isidori.
\newblock Adaptive control of linearizable systems.
\newblock \emph{IEEE Transactions on Automatic Control}, 34\penalty0
  (11):\penalty0 1123--1131, 1989.

\bibitem[Schulman et~al.(2017)Schulman, Wolski, Dhariwal, Radford, and
  Klimov]{schulman2017proximal}
John Schulman, Filip Wolski, Prafulla Dhariwal, Alec Radford, and Oleg Klimov.
\newblock Proximal policy optimization algorithms.
\newblock \emph{arXiv preprint arXiv:1707.06347}, 2017.

\bibitem[Selfridge et~al.(1985)Selfridge, Sutton, and
  Barto]{selfridge1985training}
Oliver~G Selfridge, Richard~S Sutton, and Andrew~G Barto.
\newblock Training and tracking in robotics.
\newblock In \emph{IJCAI}, pp.\  670--672, 1985.

\bibitem[Sherstov \& Stone(2005)Sherstov and Stone]{sherstov2005improving}
Alexander~A Sherstov and Peter Stone.
\newblock Improving action selection in mdp's via knowledge transfer.
\newblock In \emph{AAAI}, volume~5, pp.\  1024--1029, 2005.

\bibitem[S{\o}nderby et~al.(2016)S{\o}nderby, Caballero, Theis, Shi, and
  Husz{\'a}r]{sonderby2016amortised}
Casper~Kaae S{\o}nderby, Jose Caballero, Lucas Theis, Wenzhe Shi, and Ferenc
  Husz{\'a}r.
\newblock Amortised map inference for image super-resolution.
\newblock \emph{arXiv preprint arXiv:1610.04490}, 2016.

\bibitem[Song et~al.(2019)Song, Abdolmaleki, Springenberg, Clark, Soyer, Rae,
  Noury, Ahuja, Liu, Tirumala, et~al.]{song2019v}
H~Francis Song, Abbas Abdolmaleki, Jost~Tobias Springenberg, Aidan Clark,
  Hubert Soyer, Jack~W Rae, Seb Noury, Arun Ahuja, Siqi Liu, Dhruva Tirumala,
  et~al.
\newblock V-mpo: On-policy maximum a posteriori policy optimization for
  discrete and continuous control.
\newblock \emph{arXiv preprint arXiv:1909.12238}, 2019.

\bibitem[Sunmola \& Wyatt(2006)Sunmola and Wyatt]{sunmola2006model}
Funlade~T Sunmola and Jeremy~L Wyatt.
\newblock Model transfer for markov decision tasks via parameter matching.
\newblock In \emph{Proceedings of the 25th Workshop of the UK Planning and
  Scheduling Special Interest Group (PlanSIG 2006)}, pp.\  246--252, 2006.

\bibitem[Tamar et~al.(2013)Tamar, Xu, and Mannor]{tamar2013scaling}
Aviv Tamar, Huan Xu, and Shie Mannor.
\newblock Scaling up robust mdps by reinforcement learning.
\newblock \emph{arXiv preprint arXiv:1306.6189}, 2013.

\bibitem[Tan et~al.(2016)Tan, Xie, Boots, and Liu]{tan2016simulation}
Jie Tan, Zhaoming Xie, Byron Boots, and C~Karen Liu.
\newblock Simulation-based design of dynamic controllers for humanoid
  balancing.
\newblock In \emph{2016 IEEE/RSJ International Conference on Intelligent Robots
  and Systems (IROS)}, pp.\  2729--2736. IEEE, 2016.

\bibitem[Tanaka \& Yamamura(2003)Tanaka and Yamamura]{tanaka2003multitask}
Fumihide Tanaka and Masayuki Yamamura.
\newblock Multitask reinforcement learning on the distribution of mdps.
\newblock In \emph{Proceedings 2003 IEEE International Symposium on
  Computational Intelligence in Robotics and Automation. Computational
  Intelligence in Robotics and Automation for the New Millennium (Cat. No.
  03EX694)}, volume~3, pp.\  1108--1113. IEEE, 2003.

\bibitem[Tanaskovic et~al.(2013)Tanaskovic, Fagiano, Smith, Goulart, and
  Morari]{tanaskovic2013adaptive}
Marko Tanaskovic, Lorenzo Fagiano, Roy Smith, Paul Goulart, and Manfred Morari.
\newblock Adaptive model predictive control for constrained linear systems.
\newblock In \emph{2013 European Control Conference (ECC)}, pp.\  382--387.
  IEEE, 2013.

\bibitem[Taylor \& Stone(2009)Taylor and Stone]{taylor2009transfer}
Matthew~E Taylor and Peter Stone.
\newblock Transfer learning for reinforcement learning domains: A survey.
\newblock \emph{Journal of Machine Learning Research}, 10\penalty0 (7), 2009.

\bibitem[Tiao et~al.(2018)Tiao, Bonilla, and Ramos]{tiao2018cycle}
Louis~C Tiao, Edwin~V Bonilla, and Fabio Ramos.
\newblock Cycle-consistent adversarial learning as approximate bayesian
  inference.
\newblock \emph{arXiv preprint arXiv:1806.01771}, 2018.

\bibitem[Tobin et~al.(2017)Tobin, Fong, Ray, Schneider, Zaremba, and
  Abbeel]{tobin2017domain}
Josh Tobin, Rachel Fong, Alex Ray, Jonas Schneider, Wojciech Zaremba, and
  Pieter Abbeel.
\newblock Domain randomization for transferring deep neural networks from
  simulation to the real world.
\newblock In \emph{2017 IEEE/RSJ international conference on intelligent robots
  and systems (IROS)}, pp.\  23--30. IEEE, 2017.

\bibitem[Todorov(2007)]{todorov2007linearly}
Emanuel Todorov.
\newblock Linearly-solvable markov decision problems.
\newblock In \emph{Advances in neural information processing systems}, pp.\
  1369--1376, 2007.

\bibitem[Toussaint(2009)]{toussaint2009robot}
Marc Toussaint.
\newblock Robot trajectory optimization using approximate inference.
\newblock In \emph{Proceedings of the 26th annual international conference on
  machine learning}, pp.\  1049--1056, 2009.

\bibitem[Uehara et~al.(2016)Uehara, Sato, Suzuki, Nakayama, and
  Matsuo]{uehara2016generative}
Masatoshi Uehara, Issei Sato, Masahiro Suzuki, Kotaro Nakayama, and Yutaka
  Matsuo.
\newblock Generative adversarial nets from a density ratio estimation
  perspective.
\newblock \emph{arXiv preprint arXiv:1610.02920}, 2016.

\bibitem[Vemula et~al.(2020)Vemula, Oza, Bagnell, and
  Likhachev]{vemula2020planning}
Anirudh Vemula, Yash Oza, J~Andrew Bagnell, and Maxim Likhachev.
\newblock Planning and execution using inaccurate models with provable
  guarantees.
\newblock \emph{arXiv preprint arXiv:2003.04394}, 2020.

\bibitem[Werbos(1989)]{werbos1989neural}
Paul~J Werbos.
\newblock Neural networks for control and system identification.
\newblock In \emph{Proceedings of the 28th IEEE Conference on Decision and
  Control,}, pp.\  260--265. IEEE, 1989.

\bibitem[White(2017)]{white2017unifying}
Martha White.
\newblock Unifying task specification in reinforcement learning.
\newblock In \emph{Proceedings of the 34th International Conference on Machine
  Learning-Volume 70}, pp.\  3742--3750. JMLR. org, 2017.

\bibitem[Williams et~al.(2015)Williams, Aldrich, and
  Theodorou]{williams2015model}
Grady Williams, Andrew Aldrich, and Evangelos Theodorou.
\newblock Model predictive path integral control using covariance variable
  importance sampling.
\newblock \emph{arXiv preprint arXiv:1509.01149}, 2015.

\bibitem[Wittenmark(1995)]{wittenmark1995adaptive}
Bj{\"o}rn Wittenmark.
\newblock Adaptive dual control methods: An overview.
\newblock In \emph{Adaptive Systems in Control and Signal Processing 1995},
  pp.\  67--72. Elsevier, 1995.

\bibitem[Wulfmeier et~al.(2017{\natexlab{a}})Wulfmeier, Bewley, and
  Posner]{wulfmeier2017addressing}
Markus Wulfmeier, Alex Bewley, and Ingmar Posner.
\newblock Addressing appearance change in outdoor robotics with adversarial
  domain adaptation.
\newblock In \emph{2017 IEEE/RSJ International Conference on Intelligent Robots
  and Systems (IROS)}, pp.\  1551--1558. IEEE, 2017{\natexlab{a}}.

\bibitem[Wulfmeier et~al.(2017{\natexlab{b}})Wulfmeier, Posner, and
  Abbeel]{wulfmeier2017mutual}
Markus Wulfmeier, Ingmar Posner, and Pieter Abbeel.
\newblock Mutual alignment transfer learning.
\newblock \emph{arXiv preprint arXiv:1707.07907}, 2017{\natexlab{b}}.

\bibitem[Yu et~al.(2017)Yu, Tan, Liu, and Turk]{yu2017preparing}
Wenhao Yu, Jie Tan, C~Karen Liu, and Greg Turk.
\newblock Preparing for the unknown: Learning a universal policy with online
  system identification.
\newblock \emph{arXiv preprint arXiv:1702.02453}, 2017.

\bibitem[Zadrozny(2004)]{zadrozny2004learning}
Bianca Zadrozny.
\newblock Learning and evaluating classifiers under sample selection bias.
\newblock In \emph{Proceedings of the twenty-first international conference on
  Machine learning}, pp.\  114, 2004.

\bibitem[Zhu et~al.(2017{\natexlab{a}})Zhu, Park, Isola, and
  Efros]{zhu2017unpaired}
Jun-Yan Zhu, Taesung Park, Phillip Isola, and Alexei~A Efros.
\newblock Unpaired image-to-image translation using cycle-consistent
  adversarial networks.
\newblock In \emph{Proceedings of the IEEE international conference on computer
  vision}, pp.\  2223--2232, 2017{\natexlab{a}}.

\bibitem[Zhu et~al.(2017{\natexlab{b}})Zhu, Kimmel, Bekris, and
  Boularias]{zhu2017fast}
Shaojun Zhu, Andrew Kimmel, Kostas~E Bekris, and Abdeslam Boularias.
\newblock Fast model identification via physics engines for data-efficient
  policy search.
\newblock \emph{arXiv preprint arXiv:1710.08893}, 2017{\natexlab{b}}.

\bibitem[Ziebart(2010)]{ziebart2010modeling}
Brian~D. Ziebart.
\newblock \emph{Modeling Purposeful Adaptive Behavior with the Principle of
  Maximum Causal Entropy}.
\newblock PhD thesis, Carnegie Mellon University, 2010.

\end{thebibliography}

}
\end{document}